\documentclass{article}
\ifdefined\pdfsuppressptexinfo\pdfsuppressptexinfo=-1\fi

\usepackage[preprint]{neurips_2026}
\setcitestyle{authoryear,round}
\makeatletter
\renewcommand{\@notice}{\enlargethispage{2\baselineskip}}
\makeatother

\usepackage[T1]{fontenc}
\usepackage[utf8]{inputenc}
\usepackage[english]{babel}
\usepackage{microtype}
\usepackage{amsmath,amssymb,amsthm,amsfonts,mathtools,nicefrac}
\usepackage{graphicx,booktabs,tabularx,array,float,subfig}
\usepackage{multirow}
\usepackage{algorithm,algpseudocode,pdflscape,needspace}
\usepackage{wrapfig}
\usepackage[framemethod=tikz]{mdframed}
\usetikzlibrary{arrows.meta,positioning,calc}
\usepackage{xcolor,enumitem,etoolbox}
\usepackage{hyperref,url}
\hypersetup{
  hidelinks,
  pdftitle={Uncovering the Turing Test's Latent Signal: Indistinguishability as a Reward},
  pdfauthor={Emanuele Rimoldi, Riccardo Neumarker, Tomaso Poggio, Daniel Mitropolsky}}
\DeclareUnicodeCharacter{2212}{-}
\newcommand{\E}{\mathbb{E}}
\newtheorem{proposition}{Proposition}
\mdfdefinestyle{thesisresult}{
  linecolor=black!55,linewidth=0.55pt,roundcorner=0pt,
  backgroundcolor=white,innerleftmargin=0.75em,innerrightmargin=0.75em,
  innertopmargin=0.55em,innerbottommargin=0.45em,
  skipabove=0.7em,skipbelow=0.7em}
\surroundwithmdframed[style=thesisresult]{proposition}

\newcommand{\pown}{p_{\mathrm{ind}}}

\newcommand{\meanscore}[2]{#1\,{\scriptsize$\pm$#2}}
\newcommand{\singlescore}[1]{#1}
\newcommand{\identitymean}[2]{\meanscore{#1}{#2}}
\newcommand{\identitygain}[1]{$+$#1}

\newcolumntype{Y}{>{\raggedright\arraybackslash}X}
\newmdenv[
  backgroundcolor=blue!3,linecolor=blue!50!black,linewidth=.45pt,
  innertopmargin=6pt,innerbottommargin=6pt,
  innerleftmargin=7pt,innerrightmargin=7pt,
  nobreak=true,skipabove=4pt,skipbelow=8pt
]{promptcard}

\title{\centering Uncovering the Turing Test's Latent Signal:\\
  Indistinguishability as a Reward}%

\author{%
\parbox{\dimexpr\textwidth-4pt\relax}{\centering
\begin{tabular}{@{}c@{\hspace{1.1em}}c@{\hspace{1.1em}}c@{\hspace{1.1em}}c@{}}
Emanuele Rimoldi$^{1,2}$ & Riccardo Neumarker$^{1,3}$ & Tomaso Poggio$^{1}$ & Daniel Mitropolsky$^{1,4}$\\[4pt]
\multicolumn{4}{c}{\normalfont $^{1}$MIT \quad $^{2}$EPFL \quad $^{3}$ETH Zurich \quad $^{4}$Tufts}
\end{tabular}\\[5pt]
{\normalfont\texttt{\{erimoldi,rneumk,tpoggio,mitropol\}@mit.edu}}}}

\begin{document}

\begingroup
\setlength{\tabcolsep}{2pt}
\maketitle
\endgroup
\raggedbottom

\begin{abstract}
Can a model learn by trying to fool another model? We show that it can.
Inspired by the Turing Test and its recent generalization to arbitrary
distinguishers, we introduce a \emph{domain-based teacher-Turing-Test}, in which a
stronger teacher model serves as both the \emph{target} of imitation and the
\emph{distinguisher}, and a weaker student model is trained to make the teacher
attribute responses to itself. We use this self-attribution (or fooling) probability as
a reinforcement learning reward, without ever providing the student ground-truth
answers, nor requiring the teacher to generate responses. We test our approach
independently in four domains with verifiable answers: CommonsenseQA
(reasoning), GSM8K (mathematics), KodCode (programming), and Spider (SQL),
training a Qwen2.5-1.5B student/imitator model using DeepSeek V4 Flash as
teacher/distinguisher. Beyond improving at passing the Turing Test, the
student improves on independent, held-out evaluations in all domains and runs. To probe \emph{why}
being mistaken for the teacher can improve capability, our analyses reveal that
self-attribution is associated with correctness -- substantially more than
formatting changes -- and that it rewards \emph{partial} correctness. Our results
demonstrate a new, simple teacher-student training paradigm with a private and/or high-cost teacher. They also suggest that indistinguishability need not merely measure behavior -- it can also be a mechanism for learning.
\end{abstract}

\section{Introduction}
\label{sec:introduction}
In 1950, Alan Turing introduced the \emph{Turing Test} \textcolor{black}{\citep{turing1950computing}} -- an Artificial Intelligence (AI) system passes the Turing Test if a human judge cannot reliably distinguish, in a text-based conversation, between the AI system and another human. In this work, we ask: \textbf{can the Turing Test serve as an objective for \emph{training}?} We take as a starting point two very recent papers. The first uses reinforcement learning (RL) with a Turing-Test-based reward to train a student language model (LM) to produce human-like responses \textcolor{black}{\citep{wang2026turingrl}}. A concurrent paper introduces a novel \emph{Generalized Turing Test} (GTT) \textcolor{black}{\citep{mitropolsky2026gtt}} in which the \emph{distinguisher} and \emph{target of imitation} of the Turing Test can be an \emph{arbitrary LM} in place of a human. In this framework, an LM A ``passes the Turing Test \emph{against an LM B}'' if B cannot distinguish between itself and A imitating B.

\emph{Can the Turing Test against a stronger ``teacher'' LM be used as an RL-objective to successfully train a student model?} This contains two questions. Will the student become better at fooling the distinguisher? The second, and deeper, question is whether in doing so, the student will \emph{actually} become more intelligent in some sense, beyond simply its ability to fool the distinguisher by just mimicking its outputs?

To make the second question tractable, our idea is to restrict the Turing Test to a specific domain where answers are \emph{verifiable}. We introduce the \emph{domain-based teacher-Turing-Test}: a stronger teacher model is both the \emph{target} and \emph{distinguisher}, and a weaker student model is the imitator. We further simplify by making the Turing Test a single round -- the imitator responds to a single query. A domain-specific question set is divided into training and test components. The training signal, which we call the \emph{indistinguishability reward}  (Section~\ref{sec:turing-test-and-reward}), is the log probability that the teacher attributes the student's output to having come from its own model type.

In our main experiments, we use a small student (actor/imitator) model, Qwen2.5-1.5B, and as teacher (target/distinguisher) DeepSeek \textcolor{black}{V4 Flash.} We train the student model using GRPO, a standard RL algorithm, on our \emph{indistinguishability reward} in four domains -- \textcolor{black}{CommonsenseQA (commonsense reasoning), GSM8K (mathematics), KodCode (programming), and Spider (SQL generation)} -- for a fixed number of 8,000 RL rollouts per run, \textcolor{black}{with the optimization method of \citet{shao2024deepseekmath}}. \textbf{The student model improves on the held-out test in each domain, by \textcolor{black}{12.4, 13.0, 5.4, and 11.9} percentage points respectively} (Section \ref{sec:results}). We emphasize that ground-truth labels do not enter the RL updates; they are used only in the offline, independent evaluation of whether the reward transferred verifiable task capability. In Section \ref{sec:closed-teacher}, we present an additional variant for \textbf{fully-black-box} teachers whose logits are inaccessible; with GPT-5 Mini as distinguisher, learning to fool improves accuracy in 3 of 4 domains.

The last part of our paper asks \emph{how} this is happening.  First, we show that the distinguisher's answers are strongly correlated with correctness (a prerequisite for the Turing Test to improve a student). Additional experiments reveal that (1) changing a response's \emph{correctness} affects the reward much more strongly than adding a formatting prefix (despite the distinguisher's objective being only to determine provenance) and (2) indistinguishability rewards \emph{partial} improvements even among incorrect solutions. This latter point is why, for CommonsenseQA and Spider, our indistinguishability reward outperforms the same RL method but using the verifier's correctness bit as the reward.

Overall, our contributions are:
\begin{enumerate}
\item We introduce the \emph{domain-based teacher-Turing-Test}, and a new RL reward rooted in \emph{indistinguishability} that is label- and generation-free.
\item We train a weaker student model on four separate domains, and show improved performance on a held-out, independent evaluation in all four. %
This is evidence that the objective of the Turing Test -- fooling a distinguisher -- \emph{can} capture a meaningful signal related to capability.
\item We compare indistinguishability-reward RL to several other methods, including SFT (training directly on teacher responses), RL with a correctness oracle, and RL with \emph{distinguisher}-judged correctness. These other rewards are \emph{correctness}-based, in that either an answer or a correctness-score is provided. The Turing-based method is within 3.1 percentage points of the best method in every domain. 
\item Finally, we probe the nature of the information carried by the reward. We find a correlation between distinguisher judgments and correctness, a greater effect of correctness on the reward over formatting, and the ability of our method to reward \emph{partial} correctness.
\end{enumerate}

In addition to its implications in the theory of indistinguishability as a basis for defining intelligence, our results can also be seen more generally as a teacher-student training paradigm for \emph{black-box} or \emph{semi-black-box}, \emph{high-cost}, and/or \emph{private} settings. Namely: our teacher is a \emph{(semi)-black-box} from which only log-probabilities of two tokens, or repeated single-token outputs, are needed. 

\section{Related Work}

Alan Turing introduced the \emph{imitation game}, later known as the Turing Test, in 1950 \citep{turing1950computing}. Since then it has been hotly debated whether the Turing Test is a good definition of intelligence, i.e. that an AI system passing the Turing Test may not be intelligent in any meaningful sense (e.g. \citep{block1981psychologism,french1990subcognition,
hayes1995turing}). Our work directly builds on two recent papers: in the first use of a Turing-Test-variant as an RL reward, a model is optimized to imitate specific human-users, and they find that the method \emph{works}: the trained model fools humans and other LLMs into thinking they are human-user responses \citep{wang2026turingrl}. They compare their method to SFT, as we do in this work. Concurrently, another paper introduced the \emph{Generalized Turing Test}, the basic idea of which is that both the actor \emph{and} the distinguisher/target of the imitation are arbitrary language models \citep{mitropolsky2026gtt}. In the Discussion, they pose the enticing idea of a GAN-like \emph{training} application: train the actor to fool the distinguisher, while training the distinguisher to ask better questions and/or distinguish better. Our work makes progress on the first part of this program. \emph{This work combines the core concepts of both papers:} we use the Generalized Turing Test as an RL reward, where imitation is of a \emph{more capable} model, to ask whether training to improve fooling also improves real task capability.

Our work relates to several subfields in AI. First is the \emph{self-recognition} literature, studying the ability of models to identify outputs as their own or not; evidence is mixed -- some evaluators recognize their generations \citep{panickssery2024self}, while other studies find no consistent ability \citep{davidson2024selfrecognition}. Another connected field is student-teacher training algorithms and distillation. Common methods include sequence-level distillation \citep{kim2016seqkd,magister2022reasoning}, and on-policy distillation \citep{agarwal2024gkd}. Fine-tuning trains the model to make its own outputs indistinguishable from reference samples \citep{chen2024spin}, as opposed to the generator of those references as in the Turing Test. Our method structurally resembles other RL approaches that use feedback from a task verifier or AI judge; e.g. a correctness oracle given by a task-specific verifier  \citep{shao2024deepseekmath}, or a binary correctness judgment provided by a stronger LLM \citep{shen2026llmjudge} -- these two methods give us two of the comparison baselines presented in Section \ref{sec:experimental-setup}. As we explore in Section \ref{sec:reward-characterization}, our approach works because the distinguisher's answers are correlated with correctness. While we find evidence for this, self-preference bias in models, including exploiting this at the expense of accuracy, has been studied \citep{zheng2023judging,panickssery2024self,zhou2026convincing}.

\section{Learning from Indistinguishability}
\label{sec:turing-test-and-reward}
\subsection{Problem Formulation}
In this work, we set out to study whether the Turing Test can be used as a training method. In order to make it possible to verify whether training improves a model's capability \emph{in general}, we study a \emph{domain-restricted, one-round} Turing Test. Let $\mu$ be a \emph{domain}, by which we mean a distribution over questions for which there is a verifier $V$: given a question $x \sim \mu$ and candidate solution $y$, $V(x,y)$ outputs $1$ if $y$ is a correct solution to $x$.

 In our framework, there are two kinds of models. The first kind is a \emph{student} or \emph{actor} model $A_\theta$, parametrized by $\theta$. On input $x \sim \mu$, we interpret the output of $A_\theta(x)$ as an attempted solution to $x$.\footnote{We always take $A_\theta$ to be a general-purpose language model, but in principle, our method applies to any model $A_\theta$ that outputs solutions just on inputs from $\mu$.} 

The second kind is a \emph{teacher} or \emph{distinguisher} model $D$. Such a model has two modes: (1) a \emph{solve} mode (denoted with $\mathtt{solve}$ in the first argument) -- given $x \sim \mu$, $D(\mathtt{solve}, x)$ is an attempted solution to $x$, and (2) \emph{distinguish} mode (denoted with $\mathtt{dist}$ in the first argument) -- given $x$ and $y$, $D(\mathtt{dist},x,y) \in \{0,1\}$ is a guess as to whether $y$ was generated by $D$ on question $x$ (that is, whether $y \sim D(\mathtt{solve},x)$ or whether $y \sim A_\theta(x)$). When clear from context, we will drop the keyword argument, writing just $D(x)$ or $D(y)$. Practically, one should think of the modes being specified by fixed pre-prompts.

We now define the \textbf{\emph{domain-based teacher-Turing-Test}}%
: a question $x \sim \mu$ is sampled. Let $b \in \{0,1\}$ be the correct answer ($b$ could be randomly sampled). If $b=0$ we let $y \sim A_\theta(x)$, and if $b=1$ we let $y \sim D(\mathtt{solve}, x)$. Finally sample $g \sim D(\mathtt{dist}, x, y)$. We say the \emph{distinguisher} succeeded (or won) if $g=b$, and that the actor succeeded, or \emph{fooled} the distinguisher, if $g \neq b$. 

We denote the \emph{accuracy} of a model $M$ on $\mu$ by 
$$E(M,\mu) := \mathbb{E}_{x \sim \mu, y \sim M(x)}[V(x,y)] = \Pr_{x \sim \mu, y \sim M(x)}[V(x,y) = 1]$$ 

For the distinguisher, denote $E(D,\mu) := E(D(\mathtt{solve},\cdot),\mu)$. It will be useful to define the \emph{capability gap on} $\mu$ -- how much better the distinguisher is at solving $\mu$ than $A_\theta$: 
$$ \gamma_\mu = E(D,\mu) - E(A_\theta,\mu) $$

It is also useful to define the \emph{distinguishability gap} -- that is, how much better the distinguisher is at recognizing its own outputs than the actor's:
$$ \delta_\mu = \mathbb{E}_{x \sim \mu,~y\sim D(\mathtt{solve},x)}[D(\mathtt{dist},x,y)] - \mathbb{E}_{x \sim \mu,~y\sim A_\theta(x)}[D(\mathtt{dist},x,y)]$$

When $\gamma_\mu > 0$ and $\delta_\mu > 0$, we might hope that by training $A_\theta$ to increase the probability of fooling the distinguisher (i.e. decreasing $\delta_\mu$), $\gamma_\mu$ will also decrease. \textbf{This is the main empirical finding demonstrated by this work}. Observe that modifying $\theta$ can change only the \emph{second} term in $\delta_\mu$ above. This probability --  that of a student's output fooling the distinguisher -- is so important we give it its own name:
$$ \pown := \Pr_{x \sim \mu,~y\sim A_\theta(x)}\big[D(\mathtt{dist},x,y) = 1\big] $$

We will denote the probability of $D$ outputting 1 on a \emph{specific} question--answer pair $(x,y)$ with $\pown(x,y) := \Pr[D(\mathtt{dist},x,y) = 1 \mid x,y]$. We have $\pown = \mathbb{E}_{x\sim \mu,~y\sim A_\theta(x)}[\pown(x,y)]$.

\subsection{Reinforcement Learning with the Turing Test}
We now present a \emph{reinforcement method} based on our variant of the Turing Test. At its core, our method seeks to maximize the following \emph{unregularized} objective:
\begin{equation}
\label{equation:unregobj}
\mathbb{E}_{x \sim \mu,~y \sim A_\theta(x)}[\log \pown(x,y)] 
\end{equation} 

How is this related to $\pown$? By concavity of $\log$ (Jensen's inequality), we have $\log \pown \geq \mathbb{E}[\log \pown(x,y)]$, so our method maximizes a \emph{lower bound} on $\log\pown$, the log of the plain expected fooling probability. This is the standard maneuver in RL with a \emph{probability} reward as a kind of smoothing of values in $(0,1)$. We also empirically compare several other functions of $\pown(x,y)$ as the reward and find $\log \pown(x,y)$ gave the best improvement, see Appendix~\ref{app:reward-forms}.

Figure~\ref{fig:overview} summarizes how this reward enters the training loop.

\begin{figure}[H]
\centering
\includegraphics[width=\textwidth]{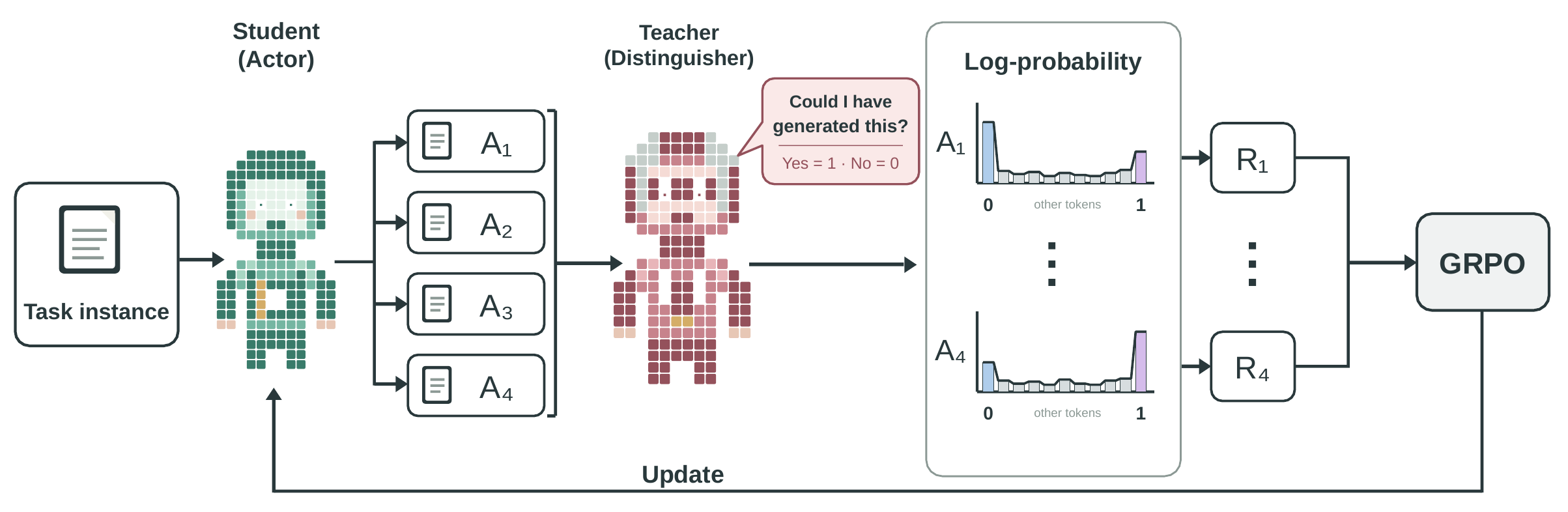}
\caption{Training with a frozen self-attribution judge. For each task, the student samples
four responses; the teacher assigns first-token log-probabilities to the binary labels 0 and 1
for each response. These yield rollout rewards for GRPO, which updates only the student. Only
the first and last response distributions are shown; no gold answers or teacher-generated
responses enter training.}
\label{fig:overview}
\end{figure}

Next we describe the method of GRPO with respect to our objective. The expectation in our objective -- equation \ref{equation:unregobj} -- is $\mathbb{E}_{x\sim\mu}\mathbb{E}_{y\sim A_\theta}$. Suppose that for every possible $x \sim \mu$, in order to approximate the inner expectation we drew $G$ responses $y_1,\ldots,y_G \sim A_{\theta_{\mathrm{old}}}(x)$ ($\theta_{\mathrm{old}}$: the parameters before the update). For each $i \in [G]$, we define the \emph{local} reward $r_i = \log\pown(x,y_i)$. In each group, rewards are standardized (subtract the group mean and divide by the group standard deviation) to give the \emph{advantages} $\hat{r}_i$. Writing $p_\theta(y|x) := \Pr[y \leftarrow A_\theta(x)]$, the following is the unregularized GRPO objective:
\begin{equation} \label{equation:batchreward}
\mathbb{E}_{x \sim \mu,~y_{1:G} \sim A_{\theta_{\mathrm{old}}}(x)}\Big[\frac{1}{G}\sum_{i=1}^G \frac{p_\theta(y_i|x)}{p_{\theta_{\mathrm{old}}}(y_i|x)}\cdot \hat{r}_i \Big]
\end{equation}

That is, GRPO generates multiple ($G$) candidate responses to a question, and weights the probability ratio of each response by its (relative) reward; intuitively, a gradient update of this objective increases the likelihood of high-reward responses. Note that when $D$ is an LLM, the probability that it outputs $0$ or $1$ in a given context is determined by the \emph{logits} $\ell_0$ and $\ell_1$, with $\pown(x,y) = e^{\ell_1} / (e^{\ell_0} + e^{\ell_1})$. This gives a particularly simple form for $r_i = \ell_1^{(i)} - \log(e^{\ell_0^{(i)}} + e^{\ell_1^{(i)}})$.

In practice, a few additional tricks are often needed for stability of gradient-based updates. Most importantly a \emph{regularizer} is added to the loss to prevent post-training from changing the student model \emph{too} much (and ceasing to be a good language model). Denote by $A_0$ the student model at the beginning of RL training. Given regularization weight $\beta$, our final objective is
\begin{equation} \label{equation:objectivefinal}
\mathbb{E}_{x \sim \mu,~y_{1:G} \sim A_{\theta_{\mathrm{old}}}(x)}\Big[\frac{1}{G}\sum_{i=1}^G \frac{p_\theta(y_i|x)}{p_{\theta_{\mathrm{old}}}(y_i|x)}\cdot \hat{r}_i \Big] - \beta \cdot \mathbb{E}_{x\sim\mu}\big[D_{\mathrm{KL}}(A_\theta(\cdot|x) \,\|\, A_0(\cdot|x))\big]
\end{equation}

In GRPO, one optimization step is a gradient-based update of Equation \ref{equation:objectivefinal}. Also, in practice $\pown(x,y)$ is clipped to $[10^{-6},1-10^{-6}]$ before the logarithm, and the whole expectation is approximated by taking a \emph{minibatch} of questions and their $G$ answers -- these details (and the final empirical form of the GRPO objective) are given in Appendix Eq.~\ref{eq:drgrpo-loss}.

\section{Experimental Setup}
\label{sec:experimental-setup}
We will now seek to \emph{empirically} answer our \textbf{main questions: (1)} if we train a student (actor) model on the indistinguishability reward $\log \pown$, will it learn to fool the distinguisher? While prior work on Turing-Test-based rewards showed this for generating human-like responses \citep{wang2026turingrl}, will it work against a teacher LLM distinguisher? \textbf{(2)} Second -- and more interestingly -- what happens to the actor's \emph{capability} on a verifiable task as it learns to fool the teacher?    

We take Qwen2.5-1.5B-Instruct as the initial student/actor $A_0$, and DeepSeek V4 Flash as teacher/distinguisher $D$. We study four domains $\mu$: CommonsenseQA (reasoning questions),
GSM8K (mathematics), KodCode (programming), and Spider (SQL generation). The student is open-weight, which allows us to train it with RL; the teacher exposes the log-probabilities of its answers; and the combination \emph{leaves room to learn}. Before any training, we verify that the teacher is much stronger than the student in every domain (that is $\gamma_\mu \gg 0$), and that as distinguisher it accepts its own responses much more often than the student's ($\delta_\mu \gg 0$) -- see Appendix Table~\ref{tab:judge-gates}.

Each domain is divided into training and evaluation sets. Separately for each domain, we train the student using GRPO on the normalized $\log\pown$ objective (Equation~\ref{equation:objectivefinal}) with two passes over $1000$ training questions and $G=4$. 

\textbf{The main thesis and finding of this work will be that (a restriction of) the Turing Test can be used in reinforcement learning to improve capability}. However, as an additional investigation we compare to training the student with three other methods, two of them RL-based. These are summarized in Table \ref{tab:baselines}:
\begin{itemize}
\item SFT, which maximizes the log-likelihood of the student generating responses sampled from the teacher
\item GRPO-quality (Quality in later tables), structurally like our Turing-Test-based method, but the teacher is asked whether the response is correct. The reward is the log-probability of the positive answer (Appendix Eq.~\ref{eq:quality-reward})
\item GRPO-verifier (Oracle), where the reward is the verifier bit $V(x,y) \in \{0,1\}$
\end{itemize}

We emphasize that every comparison method is \emph{correctness}-based, in the sense that training is done using a reference answer or score based on \emph{correctness}. All additional baselines are trained from the same starting point $A_0$ and repeated for the same 3 seeds. GRPO-based baselines use the same two passes of 1,000 questions, whereas SFT learns from 2 passes over 1,000 teacher-generated answers. 

\begin{table}[htb]
\centering
\small
\setlength{\tabcolsep}{4pt}
\renewcommand{\arraystretch}{1.1}
\caption{Summary of experimental conditions (Turing-Test-based GRPO-indisting., called Indist.\ in later tables, and comparison baselines)}
\label{tab:baselines}
\begin{tabularx}{\textwidth}{@{}lYl@{}}
\toprule
\textbf{Condition} & \textbf{Supervision} & \textbf{Reward}\\
\midrule
$A_0$ & Initial actor. & (no update)\\
SFT & Teacher answer $y_D \sim D(x)$ per question $x$. & $-\log p_\theta(y_D\mid x)$ (loss)\\
GRPO-quality & Reference-free teacher judgment of correctness. & $\log\Pr[D(\mathtt{correct?},x,y)=1]$\\
GRPO-verifier & Task verifier. & $V(x,y)\in\{0,1\}$\\
GRPO-indisting. & Teacher self-attribution (our method). & $\log\pown(x,y)$\\
\bottomrule
\end{tabularx}
\end{table}

\section{Results}
\label{sec:results}

\subsection{The Actor Learns to Fool the Distinguisher}
\label{sec:learning-to-fool}

We first answer \textbf{Main Question 1} from the beginning of Section \ref{sec:experimental-setup}. We compute the average fooling probability $\pown$ after each block of 500 updates, and find that it \textbf{continuously increases throughout RL training}. This is shown in Figure \ref{fig:checkpoint-heldout}(a), which plots $\log \pown$.

The exact improvement differs between domains. On one end of the spectrum, CommonsenseQA sees the least gain, with $\pown$ rising from 17.6\% to 22.3\% (first and last 100 updates, mean over seeds). However, on this domain $D$ rarely claims any response at all -- even $\Pr_{x\sim \mu,~y \sim D(\mathtt{solve},x)}[D(\mathtt{dist},x,y)=1]$ is only 34.7\% (Appendix Table~\ref{tab:judge-gates}). At the other end, on Spider (SQL generation) most of $A_\theta$'s queries fool $D$ from the start; $\pown$ increases from 66.0\% to 78.0\%. Overall, results confirm the intended effect of training on fooling the distinguisher: student-generated responses are more likely to be attributed to the teacher across all 12 domain-seed combinations. This works across domains with very different initial distinguisher behaviors. 

\begin{figure}[t]
\centering
\includegraphics[width=\textwidth]{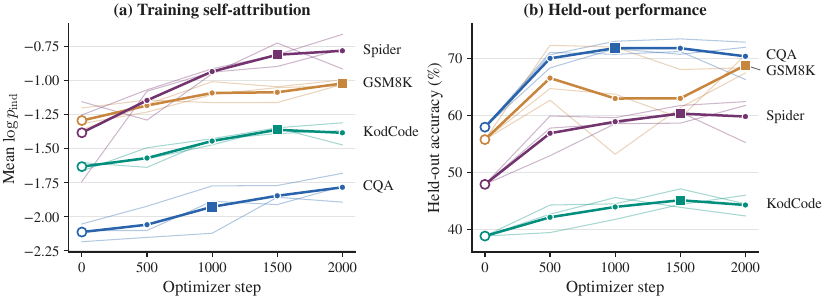}
\caption{Indist.\ training, one color per domain. \textbf{(a)} Mean $\log\pown$ of the training rollouts. \textbf{(b)} Held-out accuracy at every checkpoint. In both panels, thin lines are single seeds and bold lines
are means; squares mark the checkpoint with the highest mean held-out accuracy, which was not always the final checkpoint.}
\label{fig:checkpoint-heldout}
\end{figure}

\subsection{Accuracy Improves on the Held-Out Test Set}
\label{sec:accuracy-gains}

We now present our most surprising and significant empirical finding, which answers \textbf{Main Question 2} of Section \ref{sec:experimental-setup}. Since nothing in the framework or reward of the domain-based teacher-Turing-Test refers to correctness, a priori there is no guarantee that learning to fool the teacher will result in the student learning to solve the task. %

Figure \ref{fig:checkpoint-heldout}(b) shows that accuracy on the held-out evaluation set rises with training. Accuracy at the terminal checkpoint improves over $A_0$ in all domains (Table \ref{tab:main-results}) by 12.4 (CommonsenseQA), 13.0 (GSM8K), 5.4 (KodCode), and 11.9 (Spider) percentage points. We check that every improvement is significant under a McNemar test on paired held-out items after Holm correction over the three seeds (the weakest is KodCode seed 7, at 3.5 percentage points). 

As an aside, we find that most of the gain arrives early; after 500 updates, every domain achieves over half of its terminal gain. Figure \ref{fig:checkpoint-heldout}(b) shows which of the tested checkpoints achieved the best accuracy, often not the terminal one. While we used a fixed number of training steps for simplicity (and matching the fixed-budget protocol of \citet{wang2026turingrl}), a practical application of our method could compare checkpoints and/or use early stopping.

\begin{table}[t]
\centering
\small
\setlength{\tabcolsep}{6pt}
\caption{Held-out evaluation-set accuracy (\%) at the terminal checkpoint of the student/actor (Qwen2.5-1.5B), with mean and
standard deviation over seeds; \emph{Gain} is the improvement of Indist.\ over $A_0$ in percentage points, and bold marks the best mean in each row. %
}
\label{tab:main-results}
\begin{tabular}{@{}lrrrrrr@{}}
\toprule
& \multicolumn{3}{c}{\textbf{Not correctness-based}}
& \multicolumn{3}{c}{\textbf{Correctness-based}}\\
\cmidrule(lr){2-4}\cmidrule(lr){5-7}
\textbf{Dataset} & $\boldsymbol{A_0}$ & \textbf{Indist.} & \textbf{Gain} &
\textbf{Quality} & \textbf{SFT} & \textbf{Oracle}\\
\midrule
CommonsenseQA & \singlescore{57.9} & \identitymean{70.3}{3.6} & \identitygain{12.4} & \meanscore{\textbf{71.7}}{\textbf{1.1}} & \meanscore{71.3}{0.6} & \meanscore{69.0}{1.2}\\
GSM8K & \singlescore{55.7} & \identitymean{68.8}{1.6} & \identitygain{13.0} & \meanscore{68.9}{0.7} & \meanscore{68.6}{0.9} & \meanscore{\textbf{71.9}}{\textbf{0.6}}\\
\addlinespace[2pt]
KodCode & \singlescore{38.8} & \identitymean{44.3}{1.8} & \identitygain{5.4} & \meanscore{45.7}{1.5} & \meanscore{41.6}{0.7} & \meanscore{\textbf{47.0}}{\textbf{3.6}}\\
Spider & \singlescore{47.9} & \identitymean{59.8}{4.0} & \identitygain{11.9} & \meanscore{40.6}{35.2} & \meanscore{\textbf{61.6}}{\textbf{1.0}} & \meanscore{59.0}{4.0}\\
\bottomrule
\end{tabular}
\end{table}

\subsection{Comparison to other RL Baselines and Seed-Dependency}
The main thesis of this work is that our restriction of the Turing Test can be used for RL training, improving a student's ability to fool the distinguisher \emph{and in the same process} its general capability on a verifiable task. Our goal is not to beat other RL methods; as an additional empirical exploration, we compare with several baselines as described in Section~\ref{sec:experimental-setup}. 

As Table \ref{tab:main-results} shows, overall the Turing-Test-based method is broadly competitive. It is within 3.1 percentage points of the best method in each domain. Training initialization (the random seed) can have a strong effect; Table~\ref{tab:per-seed-results} in the appendix presents the final held-out test accuracy for each seed. Interestingly, there we see that in 4/12 runs (and in each domain other than GSM8K), indistinguishability resulted in the highest accuracy, more often than Quality and SFT; verifier-based RL (Oracle), however, was the best in 6/12 runs. Appendix~\ref{sec:outlying-runs} explains the two outlying behaviors: the Spider Quality run that scores 0\% and the collapse of two Oracle runs on CommonsenseQA to four-token answers.

The Turing-Test-based training method is fundamentally different from the compared baselines, which are all \emph{correctness-based} in one way or another. RL with verifier outputs (Oracle) does this directly. RL with a teacher's correctness score approximates this. SFT provides teacher-generated answers (which are often correct). It is interesting, we believe, that a method rooted only in \emph{indistinguishability} of output provenance can achieve effects similar to those of optimizing correctness-based objectives.

\subsection{A Closed-Weight Teacher without Log-Probabilities}
\label{sec:closed-teacher}

So far, our reward uses the teacher's log-probabilities of the two Turing-Test answers. Most proprietary and frontier models, however, do not expose them -- they only return text. Does our method need them? The Turing Test itself does not: the judge only gives a \emph{verdict}. We therefore replace DeepSeek V4 Flash with the closed-weight GPT-5 Mini and simply ask it the attribution question $K=10$ times per response. If it claims the response as its own $m$ times, the reward is $\log\big((m+1)/(K+2)\big)$, a smoothed estimate of $\log\pown$ (\textsc{MirrorVote}, Appendix~\ref{app:sampled-judgments}). Everything else is unchanged: the same actor $A_0$, the same 1,000 training questions and GRPO setup, and the same held-out evaluation as for DeepSeek. We ran a single seed per domain, so these results are preliminary.

In three of the four domains, the picture is strikingly similar. The actor learns to fool GPT-5 Mini: in every domain, the teacher's estimate of $\pown$ rises during training (Table~\ref{tab:closed-teacher}). \textbf{Again, we find that fooling the teacher improves accuracy}: by 13.3 (CommonsenseQA), 6.9 (KodCode), and 15.6 (Spider) percentage points over $A_0$ (Table~\ref{tab:closed-teacher}) -- on Spider, twice the gain of the DeepSeek run with the same seed. GSM8K is the exception, and an instructive one. GPT-5 Mini already attributed 85\% of the actor's answers to itself at the start of training, leaving little gap for the reward to close. Meanwhile, the actor drifted into writing its final answer in \LaTeX{} math, a format the evaluator does not accept: the same format drift as in the Spider Quality run of seed 23 (Appendix~\ref{sec:outlying-runs}).

\begin{table}[t]
\centering
\small
\setlength{\tabcolsep}{5pt}
\caption{Training with a \emph{closed-weight} teacher, single seed (23). $\pown$ is the mean over the training rollouts of the first and of the last 100 updates; for GPT-5 Mini it is estimated as the share $m/K$ of the ten sampled verdicts that claim the response. Acc.\ is held-out accuracy at the terminal checkpoint, and \emph{Gain} is the GPT-5 Mini run over $A_0$. The DeepSeek columns are the seed-23 run of Table~\ref{tab:per-seed-results}.}
\label{tab:closed-teacher}
\begin{tabular}{@{}lrcrcrr@{}}
\toprule
& & \multicolumn{2}{c}{\textbf{DeepSeek V4 Flash} (log-probs)} & \multicolumn{3}{c}{\textbf{GPT-5 Mini} (closed-weight, $K=10$)}\\
\cmidrule(lr){3-4}\cmidrule(lr){5-7}
\textbf{Dataset} & $\boldsymbol{A_0}$ & $\pown$: start $\to$ end & Acc. & $\pown$: start $\to$ end & Acc. & \textbf{Gain}\\
\midrule
CommonsenseQA & 57.9\% & 17.4\% $\to$ 23.0\% & 72.8\% & 69.9\% $\to$ 78.8\% & 71.2\% & $+13.3$\\
GSM8K & 55.7\% & 39.3\% $\to$ 46.0\% & 67.4\% & 84.9\% $\to$ 88.0\% & 33.1\% & $-22.7$\\
KodCode & 38.8\% & 41.6\% $\to$ 50.5\% & 46.0\% & 67.2\% $\to$ 76.9\% & 45.7\% & $+6.9$\\
Spider & 47.9\% & 69.8\% $\to$ 77.2\% & 55.2\% & 68.7\% $\to$ 77.0\% & 63.4\% & $+15.6$\\
\bottomrule
\end{tabular}
\end{table}

\section{Why does Indistinguishability Improve Accuracy?}
\label{sec:reward-characterization}

We have shown experimental evidence that training with the reward of \emph{indistinguishability}, i.e. fooling our restricted Turing Test, improves accuracy. In this section, we ask \emph{why}. Percentages are always presented in the order \textbf{CommonsenseQA, GSM8K, KodCode, Spider}.

The highest-level answer to this question is that the probability of fooling the distinguisher is highly correlated with correctness. In the next subsection, we perform interventions that provide evidence that beyond correlation, the distinguisher is truly rewarding, i.e. self-attributing, based on correctness over style/formatting.

\begin{figure}[t]
\centering
\includegraphics[width=\textwidth]{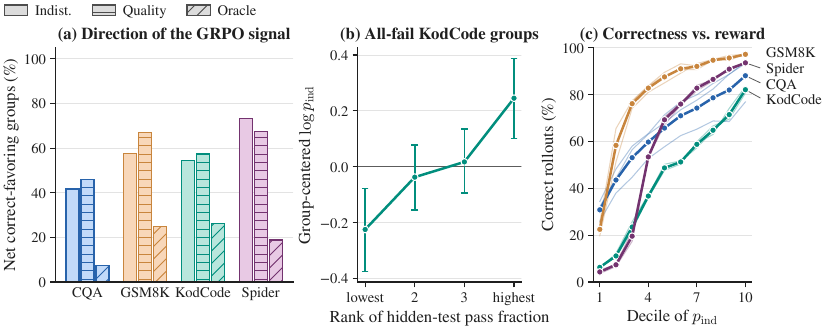}
\caption{The GRPO signal of the indistinguishability reward. \textbf{(a)} Share of the 2,000
four-response groups of a run pulled toward a correct response net of those pulled away: the
share of groups whose rewards are not all equal, times $2d-1$, where $d$ is the share of mixed
groups in which the correct responses receive the higher mean reward. Indist.\ and Oracle average
the three seeds; Quality averages the three seeds for the first factor and takes $d$ from the
Indist.\ rollouts of the three seeds, scored by the Quality judge. For Oracle $d=1$. \textbf{(b)} KodCode
groups of the seed-23 Indist.\ run in which all four programs fail the hidden tests: $\log\pown$
minus the mean of its group (not rescaled) by the within-group rank of the fraction of
hidden-test assertions passed (259 groups, 95\% bootstrap intervals). \textbf{(c)} Share of
correct Indist.\ training rollouts by decile of $\pown$, 800 rollouts per bin and seed; thin
lines are single seeds.}
\label{fig:net-useful-groups}
\end{figure}

In our training regime, $A_\theta$ produces eight answers to each training question (four in each epoch). In the seed-23 run, among ``mixed'' questions -- that is, those with at least one correct and one wrong answer -- correct responses have a higher $\log\pown$ than incorrect ones in 71.4\%, 83.0\%, 78.1\%, and 89.5\% (Appendix Figure~\ref{fig:within-prompt-gap}). The attribution score is also strongly correlated with the same teacher's judgment of correctness (mean Spearman $\rho$ over three seeds from .57 to .89; Appendix Figure~\ref{fig:judge-vs-reward}). This is visualized in Figure \ref{fig:reward-correctness}, showing the distributions of $\pown$ for incorrect and correct student answers. A very high reward is an \emph{especially} good signal of correctness, as seen in Figure \ref{fig:net-useful-groups}c: sorted into 10 bins (equal in the number of answers), the share of correct answers increases from 31\% to 88\%, 23\% to 97\%, 6\% to 82\%, and 4\% to 94\%.

What does the reward do when every response in a group is incorrect? It would be nice if $\log\pown$ rewarded \emph{partial} progress, and we find evidence that this may indeed be the case. In one of our domains, KodCode, the verifier provided with the dataset actually checks a list of assertions; $V$ outputs $1$ if all assertions pass. While some of those assertions may be more semantically important than others, as a simple baseline we take all 951 (of 2000) groups during training where all 4 programs output by $A_\theta$ fail, and re-execute the verifier's assertions \emph{separately}. Of the 259 groups with different \emph{fractions} of passed-assertions, we rank the programs within each group by this fraction and find that the reward follows it (Figure \ref{fig:net-useful-groups}b). For any two wrong programs, the program that passes more assertions has the higher reward in 61.8\% [57.8\%, 66.0\%] of pairs. This shows that even when all responses in a group are strictly incorrect, at least part of the reward signal aligns with progress towards a correct solution. This in part explains the phenomenon that $\log\pown$ can do better as a reward than the output of $V$ itself; when all solutions are wrong, the verifier-reward assigns $0$ to each answer, giving no RL signal (a \emph{sparse} reward). More generally, $\log\pown$ leaves at most 1.1\% of the groups of a run with four equal rewards (means over seeds), whereas the verifier reward leaves 74\% to 93\% of them (Figure~\ref{fig:net-useful-groups}a combines this with the direction of each reward).

\begin{figure}[t]
\centering
\includegraphics[width=\textwidth]{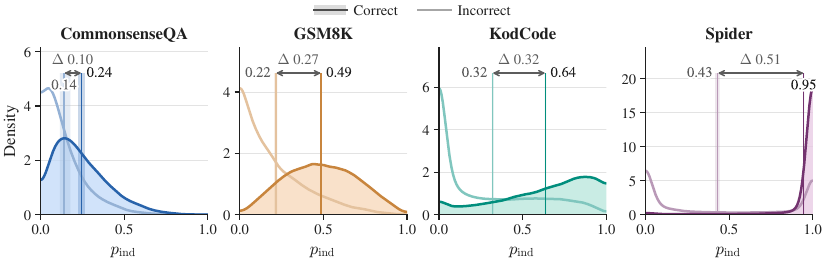}
\caption{The indistinguishability reward and verifier correctness on the Indist.\ training
rollouts: kernel density estimates of $\pown$ for correct and incorrect responses,
pooling the 8,000 rollouts of each of the three seeds (boundary-reflected, bandwidth 0.035);
vertical lines mark the conditional means and their bands span the seeds. The optimized reward is $\log\pown$; the probability scale is shown for
interpretation. KodCode labels come from re-executing the hidden tests.}
\label{fig:reward-correctness}
\end{figure}

\subsection{Counterfactual Interventions}
\label{sec:counterfactual-interventions}

For each domain, we sample 100 random \emph{correct} responses of $A_\theta$ during training. For each, we generate three \emph{variants}. The first variant changes the answer as little as possible while rendering it incorrect. A fixed rule, not a language model, makes this edit: it replaces the final answer letter (CommonsenseQA) or number (GSM8K), flips one operator or constant of the program (KodCode, e.g.\ \texttt{==} to \texttt{!=}) or one SQL component (Spider, e.g.\ \texttt{ASC} to \texttt{DESC}), and keeps the first edit that the verifier confirms as valid but incorrect (Appendix~\ref{app:counterfactual-examples}). The second variant prepends a comment-like prefix (domain-dependent) that preserves correctness. The third variant makes both edits. We then compute $\log\pown$ on all 3 variants and the original answer. 

\begin{table}[t]
\centering
\small
\setlength{\tabcolsep}{4pt}
\caption{Effects of the controlled interventions on $\log\pown$ as described in Section \ref{sec:counterfactual-interventions}. All 95\% task-instance bootstrap intervals for
semantic \( - \) style are above zero.}
\label{tab:causal-style}
\begin{tabular}{@{}lrrrr@{}}
\toprule
\textbf{Domain} & \textbf{Semantic effect} & \textbf{Style magnitude} &
\textbf{Semantic \( - \) style} & \textbf{Correct preferred}\\
\midrule
CommonsenseQA & 1.247 & 0.667 & 0.579 & 91.5\%\\
GSM8K         & 2.508 & 0.776 & 1.732 & 98.0\%\\
KodCode       & 2.813 & 0.604 & 2.210 & 95.0\%\\
Spider        & 4.277 & 0.565 & 3.712 & 100.0\%\\
\bottomrule
\end{tabular}
\end{table}

Table~\ref{tab:causal-style} reports the results of the intervention study. Denote the original response as $y_0$, the incorrect version $z_0$, and the prefixed versions $y_1$ and $z_1$ respectively. \emph{Semantic effect} is the average effect of changing just correctness -- that is, $\frac{1}{2}\sum_{i=0}^1 \big(\log\pown(y_i)-\log\pown(z_i)\big)$. \emph{Style magnitude} is the same for adding the prefix, in absolute value -- $\frac{1}{2}\sum_{w\in\{y,z\}} \left|\log\pown(w_0)-\log\pown(w_1)\right|$. Corrupting correctness lowers the reward in every domain and more than the prefix, with the correct version receiving the highest reward in 91.5\% to 100\% of comparisons.

These are but initial interventions studying ``how'' distinguishers play our Turing-Test-based game. In particular, the way we make responses incorrect in CommonsenseQA and GSM8K changes only the final answer, i.e. resulting in a response $z$ where the final answer contradicts its explanation/derivation; the distinguisher could be rejecting this contradiction as its own doing. The prefix intervention is a single, fixed line. Future work could try more sophisticated variants of actor outputs to better isolate these effects.

\section{Discussion and Future Directions}
\label{sec:iar-discussion}
To our knowledge, this is the first work to use a domain-restricted Turing Test, where imitation is of a stronger model, as a training objective. We show that on multiple domains, learning to fool the distinguisher \emph{implicitly} imparts gains in capability on that domain. Though we used the word ``teacher'' at times for the distinguisher, this teacher never generates answers nor correctness judgments; the objective is rooted only in \emph{indistinguishability}. Our main study uses but one student, two teacher models, and small domains (1,000 training questions); a central open question is how this method scales (to more models, and larger/more varied domains; see Section~\ref{sec:limitations} for a detailed description of limitations). Finally, our work initiates the study of Turing-test-based training on verifiable domains (i.e. beyond human-simulation), and as such we focused only on training the \emph{actor} -- the holy grail of Turing-test-based RL would be a GAN-like paradigm where both actor \emph{and} distinguisher are trained adversarially against one another. As a simple first next step, what if we let the distinguisher choose the next question from the verifiable domain?

\section{Limitations}
\label{sec:limitations}

We use one student (Qwen2.5-1.5B), one teacher in the main experiments (plus a closed-weight teacher in Section~\ref{sec:closed-teacher}), and 1,000 training questions per domain. The small training set is deliberate: keeping the data fixed and controlled let us spend our compute on three training seeds for every condition, which is what separates the effect of the reward from the stochasticity of RL training. Even so, outside the failed Spider Quality run, between-seed standard deviations reach 4 percentage points, comparable to most of the gaps between methods, so we read the comparison with the baselines as evidence that the method is \emph{competitive} rather than as a ranking. The closed-weight experiment of Section~\ref{sec:closed-teacher} is a proof of concept, with a single seed per domain. Given how regular the effect is with DeepSeek -- every one of the twelve runs improves over $A_0$ -- we expect it to hold across seeds, but this remains to be tested. Two further caveats concern the comparison and the teacher: SFT is not matched to the GRPO conditions in teacher budget, and the teacher's pre-training data are unknown and may include the public benchmarks we use. The latter can affect the teacher's own competence and the initial capability gap $\gamma_\mu$, though not the separation between training questions and held-out evaluation. Finally, because the distinguisher judges a response as a whole, nothing in the reward enforces the output format an evaluator expects; in a few runs the student drifted toward a format the evaluator rejects (Appendix~\ref{sec:outlying-runs}), which a light format check in the reward could prevent. We did not ablate the group size $G$, the KL coefficient, or the size of the training set.

\newcommand{\limitationsinmain}{}

\begingroup
\small
\bibliography{references}
\bibliographystyle{plainnat}
\endgroup

\clearpage
\appendix
\ifdefined\limitationsinmain\else\section{Limitations}
\label{sec:limitations}

We use one student (Qwen2.5-1.5B), one teacher in the main experiments (plus a closed-weight teacher in Section~\ref{sec:closed-teacher}), and 1,000 training questions per domain. The small training set is deliberate: keeping the data fixed and controlled let us spend our compute on three training seeds for every condition, which is what separates the effect of the reward from the stochasticity of RL training. Even so, outside the failed Spider Quality run, between-seed standard deviations reach 4 percentage points, comparable to most of the gaps between methods, so we read the comparison with the baselines as evidence that the method is \emph{competitive} rather than as a ranking. The closed-weight experiment of Section~\ref{sec:closed-teacher} is a proof of concept, with a single seed per domain. Given how regular the effect is with DeepSeek -- every one of the twelve runs improves over $A_0$ -- we expect it to hold across seeds, but this remains to be tested. Two further caveats concern the comparison and the teacher: SFT is not matched to the GRPO conditions in teacher budget, and the teacher's pre-training data are unknown and may include the public benchmarks we use. The latter can affect the teacher's own competence and the initial capability gap $\gamma_\mu$, though not the separation between training questions and held-out evaluation. Finally, because the distinguisher judges a response as a whole, nothing in the reward enforces the output format an evaluator expects; in a few runs the student drifted toward a format the evaluator rejects (Appendix~\ref{sec:outlying-runs}), which a light format check in the reward could prevent. We did not ablate the group size $G$, the KL coefficient, or the size of the training set.

\fi

\section{Datasets}
\label{app:data}

CommonsenseQA \citep{talmor2019commonsenseqa} asks five-way multiple-choice commonsense
questions; GSM8K \citep{cobbe2021gsm8k} asks grade-school math word problems with a numeric
answer; KodCode \citep{xu2025kodcode} asks for Python functions checked by hidden unit tests;
Spider \citep{yu2018spider} asks for SQL queries over unseen database schemas.

\begin{table}[H]
\centering
\small
\setlength{\tabcolsep}{4pt}
\caption{Training prompts and held-out items per domain.}
\label{tab:datasets}
\begin{tabularx}{\textwidth}{@{}lYrr@{}}
\toprule
\textbf{Dataset} & \textbf{Output and check} & \textbf{Train} & \textbf{Held-out}\\
\midrule
CommonsenseQA & Explanation + A--E; exact answer key & 1,000 & 1,221\\
GSM8K & Reasoning + number; exact normalized value & 1,000 & 1,319\\
KodCode & Python program; hidden-test execution & 1,000 & 992\\
Spider & SQL query; execution equivalence & 1,000 & 1,034\\
\bottomrule
\end{tabularx}
\end{table}

\section{Teacher Selection}
\label{app:teacher-selection}

The actor is Qwen2.5-1.5B-Instruct \citep{yang2025qwen25}. We select DeepSeek V4 Flash as teacher
and distinguisher because it returns single-token judgments with top log-probabilities and
separates its own responses from those of $A_0$ in every domain. The teacher is the April 2026
release of DeepSeek V4 Flash \citep{deepseekai2026v4} (open weights
\texttt{deepseek-ai/\allowbreak DeepSeek-V4-Flash}), queried through OpenRouter as
\texttt{deepseek/\allowbreak deepseek-v4-flash} (canonical slug
\texttt{deepseek/\allowbreak deepseek-v4-flash-20260423}), pinned to the StreamLake provider without
fallback and with reasoning disabled. Responses from any other provider are rejected, and the reward
and teacher-corpus calls also verify the reported model. For each domain, DeepSeek and
$A_0$ each answer 100 training prompts, and a judge instantiated from the same model scores the
200 responses in anonymized order with the probability that each is its own.
Table~\ref{tab:judge-gates} reports the source-attribution AUC and the paired gap in $\pown$,
an estimate $\hat\delta_\mu$ of the distinguishability gap before training. The probe prompts
come from the training data, not the held-out test, and the screen precedes any policy update.
The table also reports the capability gap $\gamma_\mu$: the teacher answers the held-out test
items with the $A_0$ prompts, greedy decoding, and at most 768 tokens, and is scored by the same
evaluators as the actor. The teacher is stronger than $A_0$ by 26 to 41 points in every domain.

\begin{table}[H]
\centering
\small
\setlength{\tabcolsep}{5pt}
\caption{Teacher selection. \emph{Capability}: held-out accuracy of the teacher and of $A_0$
under the same evaluator, and their difference $\gamma_\mu$ (percentage points).
\emph{Distinguishability}: AUC ranks 100 teacher responses against 100 $A_0$ responses; the
gap $\hat\delta_\mu$ in $\pown$ is paired over training samples.}
\label{tab:judge-gates}
\begin{tabular}{@{}lrrrcrrrr@{}}
\toprule
& \multicolumn{3}{c}{\textbf{Capability (held-out accuracy)}} & &
\multicolumn{4}{c}{\textbf{Distinguishability (training prompts)}}\\
\cmidrule(lr){2-4}\cmidrule(lr){6-9}
\textbf{Domain} & \textbf{Teacher} & $\boldsymbol{A_0}$ & $\boldsymbol{\gamma_\mu}$ & &
\textbf{AUC} & $\boldsymbol{\E_D\pown}$ & $\boldsymbol{\E_{A_0}\pown}$ & $\boldsymbol{\hat\delta_\mu}$\\
\midrule
CommonsenseQA & 83.7\% & 57.9\% & 25.8 & & .821 & .347 & .157 & .190\\
GSM8K & 96.3\% & 55.7\% & 40.6 & & .645 & .431 & .318 & .113\\
KodCode & 67.1\% & 38.8\% & 28.3 & & .754 & .630 & .351 & .279\\
Spider & 78.2\% & 47.9\% & 30.3 & & .730 & .963 & .638 & .325\\
\bottomrule
\end{tabular}
\end{table}

\section{Training and Generation Details}
\label{app:training}

\begin{table}[H]
\centering
\small
\caption{Experiment parameter configuration.}
\label{tab:training-configuration}
\begin{tabularx}{\textwidth}{@{}lY@{}}
\toprule
\textbf{Setting} & \textbf{Value}\\
\midrule
Actor & Qwen2.5-1.5B-Instruct\\
Seeds & 23, 42, 7\\
Training samples & 1,000\\
Epochs & 2\\
GRPO group size & 4\\
Rewarded rollouts per run & 8,000\\
Generation batch size & 8\\
Training batch size & 4\\
Maximum prompt / completion tokens & 4,096 / 768\\
Actor sampling & temperature .6, top-$p$ 1.0\\
Held-out decoding & greedy, at most 768 new tokens\\
Judge sampling & temperature 0, one output token, top-5 log-probabilities\\
Optimizer schedule & learning rate $3\times10^{-5}$, constant, no warm-up\\
Objective & GRPO, fixed-length loss normalization, group scaling, $\beta=.003$\\
Truncated responses & omitted from the loss\\
LoRA & rank 32, $\alpha=32$, dropout .05; attention + MLP projections\\
Training steps & 2,000 (SFT: 500)\\
\bottomrule
\end{tabularx}
\end{table}

\subsection{Implemented GRPO Objective}
\label{app:optimization-objective}

Advantages are the rewards standardized within each prompt's group of $G=4$ responses
(group mean subtracted, divided by the group standard deviation). A generation batch holds two prompts, i.e.\ eight responses,
which are shuffled and split into two optimizer minibatches of $B=4$; each response is used
once. For one minibatch the loss is
\begin{equation}
\mathcal L_{\mathrm{GRPO}}(\theta)=
-\frac{1}{B L_{\max}}\sum_{i=1}^{B}m_i\sum_{t=1}^{|y_i|}
\left[\min\!\left\{
\rho_{i,t}(\theta)\widehat A_i,
\operatorname{clip}_{1-\epsilon}^{1+\epsilon}
\!\left(\rho_{i,t}(\theta)\right)\widehat A_i
\right\}-\beta d_{i,t}(\theta)\right],
\label{eq:drgrpo-loss}
\end{equation}
with $\rho_{i,t}=\pi_\theta(y_{i,t}\mid x_i,y_{i,<t})/\pi_{\mathrm{old}}(y_{i,t}\mid x_i,y_{i,<t})$
and $d_{i,t}=\pi_{\mathrm{ref}}/\pi_\theta-\log(\pi_{\mathrm{ref}}/\pi_\theta)-1$ evaluated on the
same token, $\pi_{\mathrm{ref}}=A_0$, $m_i=0$ for truncated responses, $L_{\max}=768$,
$\epsilon=.2$, and $\beta=.003$; token probabilities use temperature $T=.6$. Since
$\pi_{\mathrm{old}}$ generated the batch, $\rho_{i,t}$ departs from one only in its second
minibatch. The fixed denominator $BL_{\max}$ is the Dr.~GRPO length correction
\citep{liu2025understanding}; group standardization is retained.

\subsection{Choice of Reward Form}
\label{app:reward-forms}

The attribution judgment can be turned into a reward in several ways, and under GRPO only the
ordering and spacing of the rewards within a group matter. Before the main experiments we compared
four forms in a short pilot: $\pown$, $\log\pown$, the logit $\ell_1-\ell_0$, and the binary verdict
$\mathbf 1\{\pown>\tfrac12\}$. Each arm trains the same actor against the same teacher, with the
learning rate, KL coefficient, and group size of Table~\ref{tab:training-configuration}, for 374
updates on GSM8K and KodCode
(seed 23), and is scored on 300 validation items (299 for KodCode), disjoint from the held-out test, after 187 and
374 updates. The selection rule, fixed in advance, maximizes the validation gain over $A_0$
averaged over the two domains. $\log\pown$ gives the largest gain in both domains
(Table~\ref{tab:reward-forms}) and is the form used throughout the paper. The binary verdict,
which keeps the decision but discards the teacher's confidence, gives the smallest mean gain.

\begin{table}[H]
\centering
\small
\caption{Reward-form pilot: validation accuracy gain over $A_0$ (percentage points; 300 validation
items for GSM8K and 299 for KodCode; best of the checkpoints after 187 and 374 updates, one seed). $A_0$ scores
60.7\% on GSM8K and 45.8\% on KodCode. Bootstrap standard errors of the mean gain are 1.8 to
2.0 points.}
\label{tab:reward-forms}
\begin{tabular}{@{}llrrr@{}}
\toprule
\textbf{Form} & \textbf{Reward} & \textbf{GSM8K} & \textbf{KodCode} & \textbf{Mean}\\
\midrule
Log probability (ours) & $\log\pown$ & $+21.0$ & $+5.7$ & $+13.3$\\
Logit & $\ell_1-\ell_0$ & $+15.0$ & $+3.7$ & $+9.3$\\
Probability & $\pown$ & $+17.7$ & $+0.3$ & $+9.0$\\
Binary verdict & $\mathbf 1\{\pown>\tfrac12\}$ & $+11.7$ & $+3.0$ & $+7.3$\\
\bottomrule
\end{tabular}
\end{table}

\section{Held-out Accuracy per Seed}
\label{app:per-seed}

Table~\ref{tab:per-seed-results} lists the held-out accuracy of each run separately.

\begin{table}[H]
\centering
\small
\setlength{\tabcolsep}{6pt}
\caption{Held-out accuracy of every training seed at the terminal checkpoint, for the runs
averaged in Table~\ref{tab:main-results}. $A_0$ is frozen and has no seed. Bold marks the best score
of each row.}
\label{tab:per-seed-results}
\begin{tabular}{@{}llrrrrr@{}}
\toprule
& & \multicolumn{2}{c}{\textbf{Not correctness-based}}
& \multicolumn{3}{c}{\textbf{Correctness-based}}\\
\cmidrule(lr){3-4}\cmidrule(lr){5-7}
\textbf{Dataset} & \textbf{Seed} & $\boldsymbol{A_0}$ & \textbf{Indist.} &
\textbf{Quality} & \textbf{SFT} & \textbf{Oracle}\\
\midrule
\multirow{3}{*}{CommonsenseQA} & 23 & \multirow{3}{*}{57.9} & \textbf{72.8} & 70.8 & 71.6 & 67.9\\
 & 42 & & 66.3 & \textbf{72.9} & 70.6 & 70.3\\
 & 7  & & \textbf{71.9} & 71.4 & 71.7 & 68.9\\
\addlinespace[2pt]
\multirow{3}{*}{GSM8K} & 23 & \multirow{3}{*}{55.7} & 67.4 & 69.7 & 68.1 & \textbf{71.3}\\
 & 42 & & 68.3 & 68.4 & 69.6 & \textbf{71.9}\\
 & 7  & & 70.6 & 68.6 & 68.2 & \textbf{72.5}\\
\addlinespace[2pt]
\multirow{3}{*}{KodCode} & 23 & \multirow{3}{*}{38.8} & \textbf{46.0} & 45.4 & 42.2 & 43.2\\
 & 42 & & 44.5 & 47.4 & 41.7 & \textbf{50.4}\\
 & 7  & & 42.3 & 44.4 & 40.8 & \textbf{47.5}\\
\addlinespace[2pt]
\multirow{3}{*}{Spider} & 23 & \multirow{3}{*}{47.9} & 55.2 & 0.0 & \textbf{60.7} & 56.6\\
 & 42 & & 62.4 & 61.5 & 62.8 & \textbf{63.5}\\
 & 7  & & \textbf{61.7} & 60.3 & 61.3 & 56.8\\
\bottomrule
\end{tabular}
\end{table}

\section{Training Diagnostics}
\label{app:training-dynamics}

Figures~\ref{fig:training-loss} and~\ref{fig:training-diagnostics} show the optimization
traces of all GRPO runs (three seeds per condition): faint lines
are the logged values of each seed and bold lines a 51-step centered average of their mean. All
runs are numerically stable: the loss oscillates around zero, clipping and gradient norms stay
small, and a few isolated gradient spikes, mostly in Quality and Oracle runs, are absorbed by the
next updates. The one marked trend is the entropy of the CommonsenseQA Oracle runs, which drops
within the first updates as two of the three seeds collapse to four-token answers
(Appendix~\ref{sec:outlying-runs}).

\begin{figure}[H]
\centering
\includegraphics[width=0.94\textwidth]{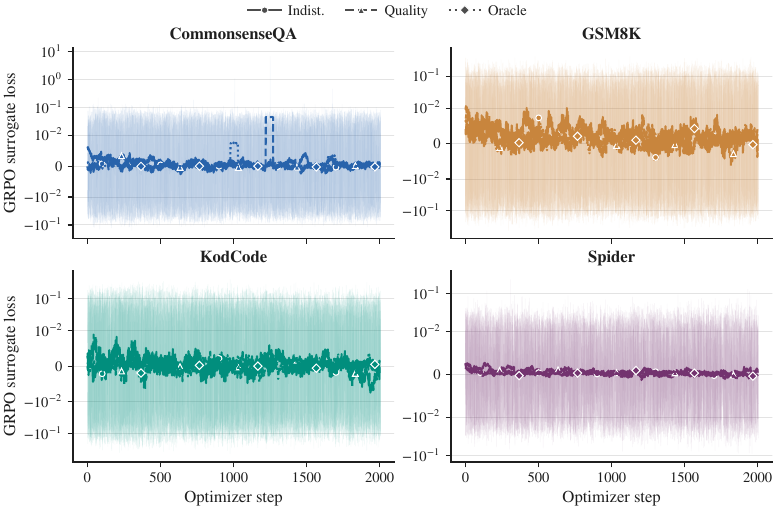}
\caption{GRPO surrogate loss of all GRPO runs (symmetric-log scale).}
\label{fig:training-loss}
\end{figure}

\begin{figure}[H]
\centering
\includegraphics[width=\textwidth,height=.78\textheight,keepaspectratio]{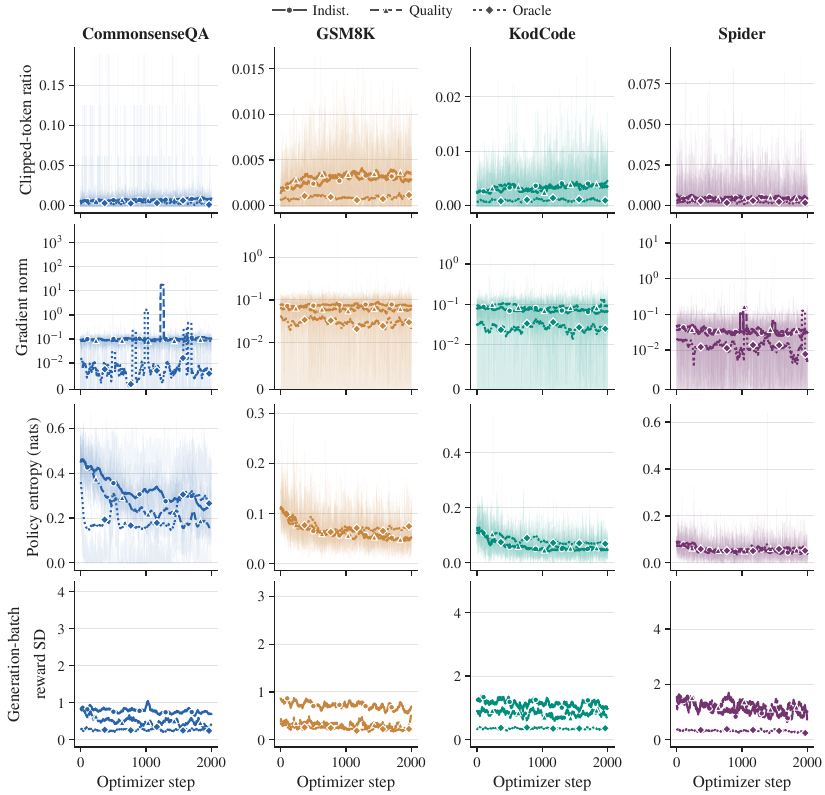}
\caption{Clipped-token ratio, pre-clipping gradient norm, token entropy, and generation-batch
reward standard deviation (rows) of all GRPO runs, by dataset (columns).}
\label{fig:training-diagnostics}
\end{figure}

\section{Counterfactual Intervention Protocol and Examples}
\label{app:counterfactual-examples}

For each domain, we select 100 correct, format-valid training responses and construct
four variants by crossing a local semantic corruption with a fixed formatting prefix:
\[
\begin{aligned}
C_0&=\text{correct original}, & C_1&=\text{correct with prefix},\\
I_0&=\text{incorrect original}, & I_1&=\text{incorrect with prefix}.
\end{aligned}
\]
Semantic edits preserve as much text as possible. The deterministic verifier must
confirm that the edited response remains format-valid and verifiable but becomes
incorrect: a parseable query or runnable program can still fail correctness checks.
We change the final choice in CommonsenseQA, the reported number in GSM8K, a local
program token or operator in KodCode, or a local SQL component in Spider. Examples
follow below.

The formatting intervention prepends the fixed, domain-specific label below without
modifying the remainder of the response.
\begin{center}
\small
\begin{tabular}{@{}ll@{}}
\toprule
\textbf{Dataset} & \textbf{Prefix added to the response}\\
\midrule
CommonsenseQA, GSM8K & \texttt{\#\#\# Candidate response}\\
KodCode & \texttt{\# Candidate response}\\
Spider & \texttt{-{}- Candidate response}\\
\bottomrule
\end{tabular}
\end{center}
The
verifier confirms that the prefix preserves the correctness label: \(C_0\) and \(C_1\)
must both remain correct, while \(I_0\) and \(I_1\) must both remain valid and incorrect. Candidate order is
hidden from the discriminator and deterministically shuffled before scoring.

Section~\ref{sec:counterfactual-interventions} defines the semantic effect, the mean of
$\log\pown(y_i)-\log\pown(z_i)$ over the two prefix conditions, and the style magnitude, the
mean of $|\log\pown(V_0)-\log\pown(V_1)|$ over the correct and incorrect versions. A positive
semantic effect favors correct responses; the style magnitude counts either direction of
prefix-induced change. Their difference compares the two effects.

The incorrect version is produced by a fixed rule. Each rule lists
candidate edits in a fixed order and keeps the first one the verifier confirms as valid but
incorrect. CommonsenseQA replaces the final answer letter with the first wrong letter (100 of
100 responses). GSM8K replaces the final number $n$ with $n+1$ (100 of 100). KodCode flips one
Python token, trying \texttt{==}/\texttt{!=}, \texttt{<}/\texttt{>}, \texttt{<=}/\texttt{>=},
\texttt{+}/\texttt{-}, \texttt{*}$\to$\texttt{+}, \texttt{and}/\texttt{or}, and
\texttt{True}/\texttt{False} in order of appearance (96 of 100); when none makes the program
fail, it inserts \texttt{raise AssertionError} at the start of the function (4 of 100). Spider
flips one SQL component, trying \texttt{ASC}/\texttt{DESC}, \texttt{AND}/\texttt{OR}, a
comparison operator, or \texttt{MAX}/\texttt{MIN} (69 of 100); otherwise it appends
\texttt{LIMIT 0} (30 of 100) or replaces the query by \texttt{SELECT 0} (1 of 100). These
Spider fallbacks change the response more than a local edit.

\subsection{Local Semantic Edits}

Unchanged task and response text is omitted; every edit is checked by the domain verifier.

\paragraph{CommonsenseQA.}
\begin{promptcard}\small
Keep the explanation unchanged; change the final choice only.\\
Correct: \texttt{Answer: B}\\
Incorrect: \texttt{Answer: A}
\end{promptcard}

\paragraph{GSM8K.}
\begin{promptcard}\small
Keep the derivation yielding 796 unchanged; change the reported answer only.\\
Correct: \texttt{Answer: 796}\\
Incorrect: \texttt{Answer: 797}
\end{promptcard}

\paragraph{KodCode.}
\begin{promptcard}\small
Keep the rest of the program unchanged; change one operator.\\
Correct: \texttt{return a + b}\\
Incorrect: \texttt{return a - b}
\end{promptcard}

\paragraph{Spider.}
\begin{promptcard}\small
Keep the SQL statement unchanged except for a result-altering clause.\\
Correct: \texttt{SELECT COUNT(*) FROM Album;}\\
Incorrect: \texttt{SELECT COUNT(*) FROM Album LIMIT 0;}
\end{promptcard}

\section{Prompts}
\label{app:baseline-prompts}

Every prompt is a single user message with no system prompt; \texttt{[TASK]} is the actor's
task prompt and \texttt{[CANDIDATE]} the actor's response.

\Needspace{20\baselineskip}
\paragraph{Quality judge.} Asks whether the candidate correctly solves the task.
\begin{promptcard}\small\ttfamily\raggedright
You are a strict task-solution correctness evaluator.

Judge only whether the candidate fully and correctly solves the supplied task. Do not judge
authorship, style, similarity to another model, or whether you would have written it. You are
not given a reference answer. Return exactly one character: 1 for a fully correct solution,
or 0 otherwise.

Task:\\{}
[TASK]

Candidate solution:\\{}
[CANDIDATE]

Label (0 or 1):
\end{promptcard}
The first-token log-probabilities of the two labels give the Quality reward, and Oracle uses
the verifier bit:
\begin{equation}
q=\frac{e^{\ell_1}}{e^{\ell_0}+e^{\ell_1}},\quad
r_{\mathrm{quality}}=\log\operatorname{clip}(q,10^{-6},1-10^{-6});
\label{eq:quality-reward}
\end{equation}
\begin{equation}
r_{\mathrm{oracle}}(x,y)=V(x,y)\in\{0,1\}.
\label{eq:oracle-reward}
\end{equation}

\paragraph{Attribution prompts.} One frozen prompt per domain asks the teacher whether it
wrote the candidate; the Spider prompt asks whether it would plausibly have produced it.

\Needspace{14\baselineskip}
\paragraph{CommonsenseQA attribution prompt.}
\begin{promptcard}\footnotesize\ttfamily\raggedright
You will be shown a commonsense multiple-choice question and one anonymized candidate response.

The response may have been written by you, or it may have been written by another\\
model answering the same question.

Your objective is not to judge whether the response is better in general.\\
Your objective is to decide whether this response was written by you, based on\\
your own reasoning choices, reasoning strategy, use of commonsense knowledge,\\
explanation structure, treatment of alternatives, and level of detail.

Task:\\{}
[TASK]

Candidate response:\\{}
[CANDIDATE]

Was this response written by you?

Your entire response must be exactly one ASCII character: 0 or 1.\\
Do not output a word, prefix, whitespace, punctuation, or explanation.

Return exactly one token:\\
1 = yes, this was written by me\\
0 = no, this was not written by me

Output:
\end{promptcard}

\Needspace{14\baselineskip}
\paragraph{GSM8K attribution prompt.}
\begin{promptcard}\footnotesize\ttfamily\raggedright
You will be shown a grade-school mathematics problem and one anonymized\\
candidate solution.

The solution may have been written by you, or it may have been written by\\
another model solving the same problem.

Your objective is not to judge whether the solution is better in general.\\
Your objective is to decide whether this solution was written by you, based on\\
your own reasoning choices, problem-solving strategy, arithmetic decomposition,\\
explanation structure, calculation choices, and level of detail.

Task:\\{}
[TASK]

Candidate solution:\\{}
[CANDIDATE]

Was this solution written by you?

Return exactly one token:\\
1 = yes, this was written by me\\
0 = no, this was not written by me

Answer:
\end{promptcard}

\Needspace{14\baselineskip}
\paragraph{KodCode attribution prompt.}
\begin{promptcard}\footnotesize\ttfamily\raggedright
You will be shown a Python programming task and one anonymized candidate solution.

The solution may have been written by you, or it may have been written by another\\
model solving the same task.

Your objective is not to judge whether the solution is better in general.\\
Your objective is to decide whether this solution was written by you, based on\\
your own coding choices, implementation strategy, Python idioms, formatting,\\
naming conventions, edge-case handling, and level of detail.

Task:\\{}
[TASK]

Candidate solution:\\{}
[CANDIDATE]

Was this solution written by you?

Return exactly one token:\\
1 = yes, this was written by me\\
0 = no, this was not written by me

Answer:
\end{promptcard}

\Needspace{14\baselineskip}
\paragraph{Spider attribution prompt.}
\begin{promptcard}\footnotesize\ttfamily\raggedright
You are DeepSeek V4 Flash. You are given a text-to-SQL task and one candidate response.

Task:\\{}
[TASK]

Candidate response:\\{}
[CANDIDATE]

Would you plausibly have produced this candidate response yourself for this task?\\
Return exactly one token: 1 for yes or 0 for no.
\end{promptcard}

\section{Additional Reward Diagnostics}

\subsection{Calibration and Controls}
\label{app:correctness-controls}

Table~\ref{tab:reward-calibration-controls} reports, for the three Indist.\ seeds, how far
$\pown$ is from a calibrated probability of correctness and how much correctness explains
$\log\pown$ beyond measured surface properties of the response. Each training rollout is
described by 22 surface features computed from its text: length (log characters, words, and
lines; mean word length), presence of the answer marker and whether it is the last line, format
and evaluator validity, truncation and near-truncation, markdown emphasis and code fences, the
fractions of bullet, numbered-step, equation, and comment lines, numeric and punctuation density,
the densities of first-person, hedging, and concluding words, and word overlap with the prompt.
Design controls are training progress, the visit (first or second epoch), and the rollout index;
prompt-level difficulty features (declared difficulty, prompt and reference length, numeric
count, reference complexity) enter the cross-validated models.

\emph{Cal.\ gap} is mean $\pown$ minus the share of correct rollouts, and \emph{ECE} the expected
calibration error over ten equal-frequency bins. \emph{Partial $r$} is the Pearson correlation
between $\log\pown$ and correctness after both are centered within each prompt and residualized,
by a ridge regression ($\alpha=1$), on the centered surface and design features.
$\Delta R^2_C$ and $\Delta R^2_F$ compare out-of-fold $R^2$ of ridge models of $\log\pown$ (five
folds grouped by prompt): $\Delta R^2_C$ is the gain from adding correctness to the difficulty,
design, and surface features, and $\Delta R^2_F$ the gain from adding the surface features to the
difficulty, design, and correctness variables.

\begin{table}[H]
\centering
\small
\setlength{\tabcolsep}{4pt}
\caption{Calibration and within-prompt controls, mean $\pm$ SD over the three Indist.\ seeds. Cal. gap is mean $\pown$ minus
correctness rate; ECE is expected calibration error. Partial $r$ is the within-prompt
correlation between correctness and $\log\pown$ after measured controls.
$\Delta R^2_C$ adds correctness after measured format and style features;
$\Delta R^2_F$ adds those features after correctness.}
\label{tab:reward-calibration-controls}
\begin{tabular}{@{}lrrrrr@{}}
\toprule
\textbf{Domain} & \textbf{Cal. gap} & \textbf{ECE} & \textbf{Partial $r$} &
\textbf{$\Delta R^2_C$} & \textbf{$\Delta R^2_F$}\\
\midrule
CommonsenseQA & $-.44\pm.09$ & $.44\pm.09$ & $.16\pm.05$ & $.14\pm.01$ & $.03\pm.03$\\
GSM8K & $-.37\pm.01$ & $.37\pm.01$ & $.38\pm.03$ & $.19\pm.01$ & $.05\pm.01$\\
KodCode & $+.02\pm.02$ & $.09\pm.01$ & $.28\pm.04$ & $.15\pm.02$ & $.02\pm.00$\\
Spider & $+.15\pm.00$ & $.16\pm.01$ & $.40\pm.03$ & $.20\pm.01$ & $.01\pm.01$\\
\bottomrule
\end{tabular}
\end{table}

\subsection{Within-Prompt Reward Gap}
\label{app:within-prompt-gap}

Figure~\ref{fig:within-prompt-gap} compares, for every prompt whose eight seed-23 training
rollouts include both correct and incorrect responses, the mean $\log\pown$ of its correct and
incorrect rollouts. The correct ones receive the higher mean reward in 71.4\%, 83.0\%, 78.1\%,
and 89.5\% of these prompts.

\begin{figure}[H]
\centering
\includegraphics[width=\textwidth]{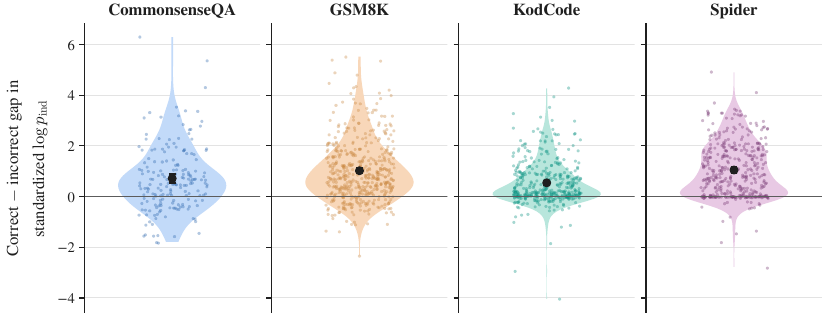}
\caption{One point per prompt whose eight seed-23 rollouts include both correct and incorrect
responses: mean $\log\pown$ of the correct rollouts minus that of the incorrect ones, after
standardizing $\log\pown$ within each domain. Black: mean gap and its 95\% bootstrap interval.}
\label{fig:within-prompt-gap}
\end{figure}

\subsection{Reward and the Teacher's Correctness Judgment}
\label{app:judge-vs-reward}

We scored the 8,000 training rollouts of each Indist.\ seed and domain a second time with the
Quality prompt, which asks the same teacher whether the response is correct, and compared the
two log-probabilities (Figure~\ref{fig:judge-vs-reward}). They move together, weakly on
CommonsenseQA and strongly on Spider. Averaged over the three seeds, Spearman $\rho$ is .59, .57,
.81, and .89, and $\log q$ explains 31\%, 41\%, 63\%, and 77\% of the variance of $\log\pown$
(5\%, 29\%, 34\%, and 53\% within prompts). As detectors of correct responses, $\pown$ reaches an
AUC of .72, .83, .78, and .88 and $q$ of .80, .89, .80, and .87: the correctness judgment is ahead
by .06 to .08 on CommonsenseQA and GSM8K, and the two are within .02 of each other on KodCode and
Spider. The seeds agree closely: between-seed standard deviations are at most .03 for $\rho$ and
$R^2$ and at most .05 for the AUCs.

\begin{figure}[H]
\centering
\includegraphics[width=\textwidth]{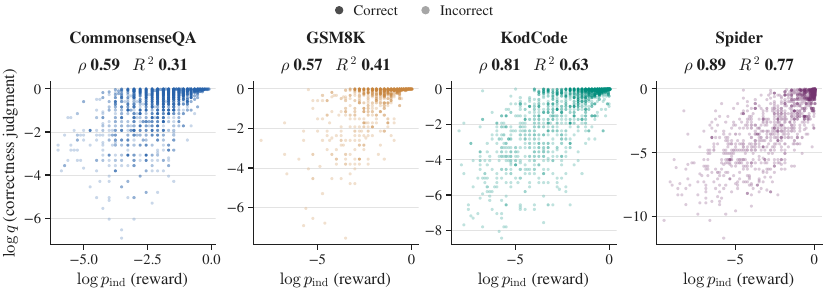}
\caption{$\log\pown$ against $\log q$, the log-probability that the same teacher, asked whether
the response is correct, assigns to ``correct'', on the training rollouts of the three Indist.\
seeds (2,502 points per domain, drawn equally from each seed); titles give Spearman $\rho$ and the
$R^2$ of a linear fit, averaged over seeds.}
\label{fig:judge-vs-reward}
\end{figure}

\subsection{Repeated Training Prompts}
\label{app:group-relative-signal}

Each prompt is visited once in each epoch. From the first to the second visit, mean $\pown$
rises in all twelve Indist.\ runs, by .021 to .057, while verifier correctness rises in ten of
them and falls in two, by 3.9 points on GSM8K (seed 23) and 0.5 on KodCode (seed 42). On Spider,
mean $\pown$ of incorrect responses rises from .36--.40 to .47--.49 in every seed while that of
correct responses stays near .95: a higher reward does not guarantee higher correctness.

\section{Outlying Runs}
\label{sec:outlying-runs}

\paragraph{Spider, Quality seed 23 (0.0\%).} This run is an instance of format drift. The
evaluator accepts a bare query or a single isolated code fence, and 1,030 of the run's 1,034
held-out responses instead wrap a fenced SQL query in prose. The drift emerges during training:
none of the first 400 rollouts has this format and 391 of the last 400 do, while the Quality
judge gives 270 of these 391 rollouts $q>\tfrac12$. The judged reward does not rise during the
drift ($-1.10$ to $-1.41$ over the first and last 400 rollouts), so the judge simply fails to
penalize a format the evaluator rejects.
The drift is visible in the response length (Figure~\ref{fig:response-length}): this run grows from
about 30 to 67 tokens, while every other Spider run stays between 21 and 41.

\paragraph{CommonsenseQA, Oracle.} Two of the three Oracle seeds collapse to four-token answers
(\texttt{Answer: X}): seed 23 falls from 82 tokens to at most five by step 15, seed 7 collapses
and recovers after step 1,400, and seed 42 never collapses. Four such answers to one question
almost always agree, so 87\% to 98\% of the groups of seed 23 have no reward contrast and its
token entropy falls from .118 to .022. Seed 23 ends at 67.9\%, against 68.9\% and 70.3\% for
the other seeds.

\begin{figure}[H]
\centering
\includegraphics[width=\textwidth]{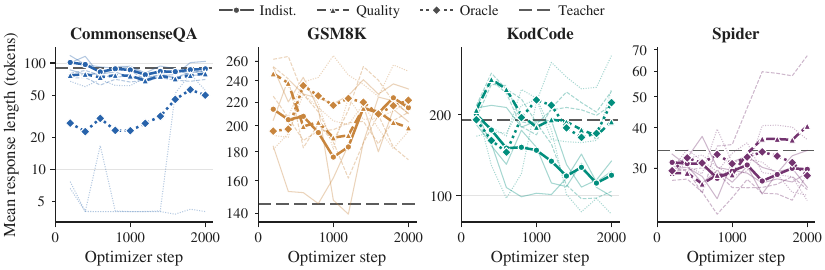}
\caption{Mean response length of the training rollouts under each reward, in 200-update bins;
thin lines are single seeds and bold lines their mean. The dashed line is the mean length of the
teacher's own 1,000 completions per domain.}
\label{fig:response-length}
\end{figure}

\section{Relations Between the Source-Attribution Tests}
\label{app:test-reductions}
\textcolor{black}{The following reductions assume equal source priors and uniformly
randomized pair order. Generation and judgment use separate, stateless calls, and a fresh
reference is drawn independently from the same teacher distribution, conditional on the task.
All calls receive the same task; its notation is suppressed in these reductions.}

\begin{proposition}[From one-message to paired attribution]
\label{prop:single-to-pair}
Suppose a LM $D$ distinguishes the 1-message variant with probability $\geq \frac{1}{2} + \epsilon$. Then, with one call to $D$ one can distinguish the paired-message variant with probability $\geq \frac{1}{2} + \epsilon$.
\end{proposition}
\begin{proof}
Let $D(m)$ be the random variable of $D$'s output in the 1-message variant on response $m$. Given paired responses $m_1,m_2$, call $D(m_1)$ and output $(D,A)$ if it returns $1$, and $(A,D)$ otherwise.

Under uniform pair order, its success probability is
\[
\tfrac12\Pr[D(m_D)=1]+\tfrac12\Pr[D(m_A)=0]
\geq\tfrac12+\epsilon
\]
by the one-message hypothesis.
\end{proof}

\textcolor{black}{One can instead query the one-message judge on both responses.
If its labels differ, $(1,0)$ or $(0,1)$, return $(D,A)$ or $(A,D)$ respectively;
if they coincide, choose the assignment uniformly at random. This particular rule has
the same success probability as the one-query reduction:}
\begin{proof}
{\color{black}
Let $U=D(m_D)$ and $V=D(m_A)$ be the two binary verdicts. Then
\begin{align*}
\Pr[\mathrm{success}]
&=\Pr[U=1,V=0]+\tfrac12\Pr[U=V]\\
&=\tfrac12\bigl(\Pr[U=1]+\Pr[V=0]\bigr),
\end{align*}
which is the success probability of the one-message judge under equal source priors.
This identity does not require independence between $U$ and $V$.}
\end{proof}

\begin{proposition}[From paired to one-message attribution]
\label{prop:pair-to-single}
\textcolor{black}{Suppose a LM $D$ distinguishes the paired-message variant with
probability $\geq \frac{1}{2}+\epsilon$. Then one judgment call and one independent
\emph{generation} call to $D$ suffice to distinguish the 1-message variant with probability
$\geq \frac{1}{2}+\frac{\epsilon}{2}$.}
\end{proposition}
\begin{proof}
\textcolor{black}{We are given a response $m$ with an unknown source
$S\in\{D,A\}$, under equal source priors. Independently generate $m_D$ from $D$ and draw
$b\sim\operatorname{Bernoulli}(1/2)$. If $b=0$, query $D(m,m_D)$; otherwise query
$D(m_D,m)$. If the judge selects the position of our local reference $m_D$, output $A$;
otherwise output $D$.}

\textcolor{black}{If $S=D$, $m$ and the independently generated $m_D$
are identically distributed conditional on the task. By exchangeability and the uniform
position bit, the stateless judge selects our local reference with probability exactly
$1/2$, averaging over the two samples and $b$. This need not hold conditional on a fixed
position for a position-biased judge. Hence the reduction correctly outputs $D$ with
probability $1/2$ in this case.}

\textcolor{black}{If $S=A$, this is exactly the randomized paired-message
experiment. The reduction outputs $A$ whenever the judge correctly selects $m_D$, which
occurs with probability at least $1/2+\epsilon$.}

Overall we have that, in the 1-message variant, we distinguish with probability
\begin{align*}
&\tfrac12\Pr[\mathrm{correct}\mid S=D]+\tfrac12\Pr[\mathrm{correct}\mid S=A] \\
&\geq \frac{1}{2}\cdot \frac{1}{2} + \frac{1}{2}\Big(\frac{1}{2}+\epsilon\Big) \\
&= \frac{1}{2} + \frac{\epsilon}{2}
\end{align*}
\end{proof}

These are constructions of derived distinguishers, not equivalences between
the scores of independently prompted native judges.

\section{Alternative Reward Protocols}
\label{app:reward-extensions}

These reward oracles are alternatives to the log-probability reward of Section~\ref{sec:turing-test-and-reward}, not new policy optimizers.
They retain the batch and protocol contracts of Section~\ref{sec:turing-test-and-reward}.
\textsc{MirrorVote} was used to train the actor against a closed-weight teacher
(Section~\ref{sec:closed-teacher}); the generated-reference reward has not been evaluated in
policy training, and its subsection reports only the separate two-order diagnostic that was
executed.

\subsection{Reward Estimation from Repeated Binary Judgments}
\label{app:sampled-judgments}

If label log-probabilities are unavailable, an endpoint exposing sampled binary
labels can be queried $K$ times for a fixed response, without weight access. Here $K$ counts
independent judge calls, not the $G$ actor responses in a GRPO group. These judgments
give $\widehat p=m/K$, where $m$ counts own-source labels and $p$ is their
sampling probability, with
\begin{equation}
\E[\widehat p]=p,\qquad \operatorname{Var}(\widehat p)=\frac{p(1-p)}{K}\leq\frac1{4K}.
\end{equation}
For the $K=10$ setting, the unsmoothed estimator has worst-case
standard error approximately .158 and
it has only 11 unsmoothed levels, so nearby candidates can be misranked frequently. Section~\ref{sec:closed-teacher} evaluates this protocol with $K=10$ and GPT-5 Mini as the
teacher, one training seed per domain.

\paragraph{Protocol used with GPT-5 Mini.} The teacher is \texttt{openai/gpt-5-mini}, served
by OpenAI through OpenRouter without fallback to other providers, at the endpoint's default
sampling temperature (1.0) and minimal reasoning effort. Its output is constrained to a JSON
object with a single binary field, and no log-probabilities are requested. Each rollout receives
ten independent verdicts, so a full run of 8,000 rollouts makes 80,000 judgments. Transient
HTTP errors (rate limits and gateway errors) are retried up to eight times; a response that
does not parse as a valid label stops training rather than being resampled. Training resumed
from the last saved checkpoint when a run was interrupted by API errors. One prompt serves all
four domains:
\begin{promptcard}\footnotesize\ttfamily\raggedright
You will be shown one task and one anonymized candidate response.

The response may have been written by you, or by another model solving the same task. Decide
only whether it could plausibly be your own response.

Task:\\{}
[TASK]

Candidate response:\\{}
[CANDIDATE]

Return exactly one JSON object with no explanation: \{"label": 1\} if it could plausibly be
your response, or \{"label": 0\} otherwise.
\end{promptcard}

\begin{algorithm}[H]
\caption{\textsc{MirrorVote}: reward from repeated binary judgments}
\label{alg:closed-logitless-iar}
\begin{algorithmic}[1]
\color{black}
\State \textbf{Input:} batch $\mathcal B$; fixed endpoint $D_{\rm API}$ and protocol $(T_D,\mathcal N,\kappa)$; integer $K\geq1$
\State \textbf{Output:} terminal rewards $\mathbf r=(r_1,\ldots,r_{n_B})$
\State Initialize $\mathbf r\gets(0,\ldots,0)$
\For{$i=1,\ldots,n_B$}
    \State Initialize teacher-label count $m_i\gets0$
    \For{$j=1,\ldots,K$}
        \State Independently sample and parse $z_{i,j}$ from $D_{\rm API}(T_D(x_i,\mathcal N(y_i)))$ under $\kappa$, without requesting label log-probabilities
        \If{$z_{i,j}\notin\{0,1\}$ or the endpoint contract fails}
            \State \textbf{Abort} without returning a reward vector
        \EndIf
        \State $m_i\gets m_i+z_{i,j}$
    \EndFor
    \State $r_i\gets\log\!\left((m_i+1)/(K+2)\right)$
\EndFor
\State \Return $\mathbf r$
\end{algorithmic}
\end{algorithm}

\textcolor{black}{\textsc{MirrorVote} estimates the probability
$p_\kappa=\Pr_\kappa(Z=1\mid x,y)$ of the deployed label sampler. This equals the
unclipped two-label probability underlying $\pown$ only if the two decoding laws agree.
For independent valid labels, $m\sim\operatorname{Binomial}(K,p_\kappa)$ and
$\E[(m+1)/(K+2)]=(Kp_\kappa+1)/(K+2)$: smoothing and the logarithm make the finite-$K$
reward a biased estimator of $\log p_\kappa$. Repeating a deterministic verdict cannot
recover a graded probability. Semantic contract violations stop scoring rather than being
resampled until a valid label appears.}

\subsection{Generated-Reference Test and Binary Reward}
\label{app:generated-reference}

A \emph{generated-reference} test draws an independent teacher response
$Y_D\sim\pi_D(\cdot\mid x)$ and compares it with an actor response $Y_A$.
Let $B$ be the uniformly randomized teacher position: $C_B=Y_D$ and
$C_{3-B}=Y_A$, with $B\in\{1,2\}$. The judge selects a purported teacher
position $\widehat B$. For the displayed pair, the actor's favorable score is
\begin{equation}
p_{\mathrm{gen}}(x,y_A,y_D;B)
=\Pr_D(\widehat B\neq B\mid x,C_1,C_2).
\label{eq:generated-reference}
\end{equation}
An order-averaged score averages this quantity over the two positions.
The one-message, or \emph{latent-reference}, formulation instead asks for a source
label given only $(x,y)$. Under equal source priors and independent, stateless
teacher calls, the two tests admit the reductions with explicit query costs in
Appendix~\ref{app:test-reductions}. These constructions do not imply equivalent
scores or training signals from native one-message and paired prompts.

The two-order attribution probe below was executed; policy training with
\textsc{MirrorDuel} was not. Its continuous scores and the proposed binary
reward are distinct experimental objects.

\subsubsection{Order Sensitivity and Source Attribution}
\label{sec:pairwise}

We also evaluate the generated reference test in Eq.~\ref{eq:generated-reference}, a more
literal imitation game. An independent DeepSeek call produces and caches $y_D$ before seeing
the actor answer; a fresh stateless DeepSeek call then identifies which of the two anonymous
responses it is more likely to have generated. Each actor and reference pair is judged in both
orders. The judge is instructed to solve privately and return candidate index 0 or 1.
Let $p_A^{\rm first}$ ($p_A^{\rm second}$) be the normalized first token probability
assigned to the actor's candidate position when that candidate is presented first (second).
Thus both quantities are actor attribution probabilities, rather than $\pown$ from the
latent reference test. Their difference
$\Delta_{\rm order}=p_A^{\rm first}-p_A^{\rm second}$ is the paired order effect, while
$p_{\rm pair}=\tfrac12(p_A^{\rm first}+p_A^{\rm second})$ is the order averaged score.

\begin{table}[H]
\centering
\small
\setlength{\tabcolsep}{4pt}
\caption{Generated-reference probe: 100 prompts, four actor rollouts per prompt,
two orders, and 800 judge calls per domain. CQA denotes CommonsenseQA. The upper
block reports the mean and SD of the pair-level order-averaged actor attribution
score and the means for each order. The lower block stratifies the paired order
effect by deterministic correctness; $n$ counts actor--teacher pairs and bootstrap
resampling is by task. Bold marks the largest overall and stratified order effects.}
\label{tab:pairwise}
\label{tab:pairwise-correctness}
\begin{tabular}{@{}lrrrrrrrr@{}}
\toprule
\multicolumn{9}{@{}l}{\textbf{Overall Actor Attribution}}\\
\midrule
\textbf{Statistic} & \multicolumn{2}{c}{\textbf{CQA}} & \multicolumn{2}{c}{\textbf{GSM8K}} & \multicolumn{2}{c}{\textbf{Spider}} & \multicolumn{2}{c}{\textbf{KodCode}}\\
\midrule
Mean $\bar p_{\rm pair}$ & \multicolumn{2}{r}{0.414} & \multicolumn{2}{r}{0.438} & \multicolumn{2}{r}{0.394} & \multicolumn{2}{r}{0.428}\\
SD of $p_{\rm pair}$ & \multicolumn{2}{r}{0.109} & \multicolumn{2}{r}{0.107} & \multicolumn{2}{r}{0.128} & \multicolumn{2}{r}{0.103}\\
$p_A^{\rm first}$ & \multicolumn{2}{r}{0.734} & \multicolumn{2}{r}{0.770} & \multicolumn{2}{r}{0.665} & \multicolumn{2}{r}{0.773}\\
$p_A^{\rm second}$ & \multicolumn{2}{r}{0.095} & \multicolumn{2}{r}{0.107} & \multicolumn{2}{r}{0.124} & \multicolumn{2}{r}{0.083}\\
Order effect $\Delta_{\rm order}$ & \multicolumn{2}{r}{0.639} & \multicolumn{2}{r}{0.663} & \multicolumn{2}{r}{0.541} & \multicolumn{2}{r}{\textbf{0.690}}\\
\midrule
\multicolumn{9}{@{}l}{\textbf{Order Effect by Correctness of $(D,A)$}}\\
\midrule
\textbf{Correctness} & \multicolumn{2}{c}{\textbf{CQA}} & \multicolumn{2}{c}{\textbf{GSM8K}} & \multicolumn{2}{c}{\textbf{Spider}} & \multicolumn{2}{c}{\textbf{KodCode}}\\
\cmidrule(lr){2-3}\cmidrule(lr){4-5}\cmidrule(lr){6-7}\cmidrule(l){8-9}
$(D,A)$ & \textbf{Mean} & $\boldsymbol{n}$ & \textbf{Mean} & $\boldsymbol{n}$ & \textbf{Mean} & $\boldsymbol{n}$ & \textbf{Mean} & $\boldsymbol{n}$\\
\midrule
Both correct & 0.691 & 176 & 0.702 & 206 & 0.633 & 165 & \textbf{0.734} & 141\\
$D$ correct, $A$ wrong & 0.593 & 156 & 0.627 & 170 & 0.455 & 147 & 0.656 & 155\\
$D$ wrong, $A$ correct & 0.666 & 23 & 0.563 & 4 & 0.675 & 23 & 0.710 & 11\\
Both wrong & 0.578 & 45 & 0.582 & 20 & 0.452 & 65 & 0.676 & 93\\
\bottomrule
\end{tabular}
\end{table}

The same response receives 64 to 66 percentage points more actor attribution probability
merely by appearing first (95\% intervals: [61.8, 65.9] pp on CQA and [64.3, 68.2] pp on
GSM8K). This pattern is not explained by response correctness.
The lower block of Table~\ref{tab:pairwise} stratifies the cached pairs by
deterministic correctness of the teacher and actor responses. The order effect
persists when both candidates are correct, when only one is correct, and when both are wrong;
the reverse mixed stratum has few pairs and is consequently descriptive. The same two-order
analysis on Spider and KodCode reproduces this pattern. Across the same calls, binary-mass
coverage---the share assigning at least 90\% probability mass to the two permitted labels---is
49.6\% on CQA and 60.1\% on GSM8K. Emitted-label agreement with the binary argmax is
86.4\% and 87.6\%, respectively. The pairwise configuration therefore does not provide a
stable training signal. This result concerns this DeepSeek prompt
and label configuration rather than pairwise discriminators in general. Ordering bias has
also been documented in model self-recognition settings
\citep{panickssery2024self,davidson2024selfrecognition}.

\subsubsection{MirrorDuel: A Generated-Reference Binary Reward}

\textcolor{black}{\textsc{MirrorDuel} has a separate setup stage. Before actor
rollouts, generate one reference $C_D(x)\sim\pi_D(\cdot\mid x)$ per training prompt using
only that prompt and frozen generation settings; retain it for every response group for
that prompt. A separate pilot uses $n\geq1$ independent task-level replicates with
independent reference draws, judging each
actor--reference pair in both orders with fresh judge calls. Let $w_{j,k}^{\rm first}$
and $w_{j,k}^{\rm second}$ indicate that the actor is selected in these two orders.
The audited pilot record $\mathcal H$ retains the raw binary outcomes and frozen protocol
identifiers, together with
$d_j=G^{-1}\sum_{k=1}^{G}(w_{j,k}^{\rm first}-w_{j,k}^{\rm second})\in[-1,1]$ and
$v_j\in\{0,1\}$, where $v_j=1$ iff at least one fixed order gives both reward values
within group $j$. Invalid pilot calls invalidate the record. The protocol, pilot size $n$,
tolerance $0<\tau<1$, and error level $0<\alpha<1$ are fixed before this pilot,
separately per domain; subsequent protocol tuning requires fresh pilot data.
Its confidence radius counts independent task-level replicates, not correlated responses
within a replicate.}

\begin{algorithm}[t]
\caption{\textsc{MirrorDuel}: generated-reference binary reward}
\label{alg:binary-duel-iar}
\begin{algorithmic}[1]
\color{black}
\State \textbf{Input:} batch $\mathcal B$; frozen pairwise protocol $(D,T_D^{\rm pair},\mathcal N,\kappa)$; reference cache $C_D$; audited pilot record $\mathcal H$; $\tau,\alpha\in(0,1)$
\State \textbf{Output:} binary rewards $\mathbf r\in\{0,1\}^{n_B}$
\State Read $n$ task-level summaries $(d_j,v_j)$ and protocol identifiers from $\mathcal H$
\If{$n=0$, pilot records are invalid or protocol-mismatched, or $C_D$ is invalid or incomplete}
    \State \textbf{Abort} without returning a reward vector
\EndIf
\State $\widehat\Delta\gets n^{-1}\sum_{j=1}^{n}d_j$; $h_n\gets\sqrt{2\log(2/\alpha)/n}$
\If{$|\widehat\Delta|+h_n>\tau$ or $\sum_{j=1}^{n}v_j=0$}
    \State \textbf{Abort}: pilot evidence is insufficient to enable this reward
\EndIf
\State Initialize $\mathbf r\gets(0,\ldots,0)$
\State Draw independent $b_g\sim\operatorname{Bernoulli}(1/2)$ once per rollout group $g$, independently of its responses
\For{$i=1,\ldots,n_B$}
    \State $(a_i,t_i)\gets(\mathcal N(y_i),\mathcal N(C_D(x_i)))$
    \If{$b_{g_i}=1$}
        \State $(c_{i,1},c_{i,2})\gets(a_i,t_i)$
    \Else
        \State $(c_{i,1},c_{i,2})\gets(t_i,a_i)$
    \EndIf
    \State Query a fresh, stateless $D$ judge on $T_D^{\rm pair}(x_i,c_{i,1},c_{i,2})$ under $\kappa$
    \State Parse $z_i\in\{0,1\}$: $1$ selects the first response, $0$ the second
    \If{the pinned endpoint or binary-label contract fails}
        \State \textbf{Abort} without returning a reward vector
    \EndIf
    \State $r_i\gets\mathbf{1}\{z_i=b_{g_i}\}$
\EndFor
\State \Return $\mathbf r$
\end{algorithmic}
\end{algorithm}

\textcolor{black}{\textsc{MirrorDuel} uses one binary judgment per actor response
after reference generation and the pilot. The group shares its order bit: a judge that
always selects the first position then gives a constant reward within that group and hence
zero relative advantage. Independent per-response orders would instead introduce spurious
within-group reward variation. This control does not remove content-dependent position
effects. Under the stated pilot assumptions, Hoeffding's inequality bounds the absolute mean order effect
by $|\widehat\Delta|+h_n$ with probability at least $1-\alpha$; this is a statement about
the pilot distribution, not correctness alignment or later actor distributions. The
nondegeneracy check likewise does not prove that the reward contains useful task information.}

\textcolor{black}{For a fixed actor--reference pair, the expectation of the binary
reward is one half the sum of the probabilities of selecting the actor in the two orders,
under the actual verdict law $\kappa$. It need not equal the normalized-log-probability
score $p_{\rm pair}$ reported in Appendix~\ref{sec:pairwise}, especially for a greedy judge.
The proposed binary pilot and training arm have not been run. The reported continuous
two-order probe shows severe position bias and does not validate the proposed training
algorithm.}

\end{document}